\documentclass{article}

\usepackage{etoolbox}
\usepackage{comment}
\newtoggle{colt}
\togglefalse{colt}
\newcommand{\colt}[1]{\iftoggle{colt}{#1}{}}
\newcommand{\arxiv}[1]{\iftoggle{colt}{}{#1}}

\newcommand{\sumt}[1][T]{\sum_{t=1}^{#1}}
\newcommand{\yhat}{\wh{y}}

\newcommand{\cov}{V} %

\newcommand{\xt}[1][t]{x\ind{#1}}
\newcommand{\zt}[1][t]{z\ind{#1}}
\newcommand{\yt}[1][t]{y\ind{#1}}
\newcommand{\yht}[1][t]{\yhat\ind{#1}}
\newcommand{\htht}[1][t]{\hth\ind{#1}}
\newcommand{\covt}[1][t]{\cov\ind{#1}}
\newcommand{\lcovt}[1][t]{\wt{\cov}\ind{#1}}
\newcommand{\St}[1][t]{S\ind{#1}}
\newcommand{\inv}{^\dagger}

\newcommand{\Ccov}{C_{\mathsf{cov}}}

\newcommand{\Env}{\mathsf{Env}}

\arxiv{
\usepackage[letterpaper, left=1in, right=1in, top=1in,bottom=1in]{geometry}
  \usepackage{parskip}
  \usepackage[colorlinks=true, linkcolor=blue!70!black, citecolor=blue!70!black,urlcolor=black,breaklinks=true]{hyperref}
  \usepackage[dvipsnames]{xcolor}

\usepackage[utf8]{inputenc} %
\usepackage[T1]{fontenc}    %
\usepackage{lmodern}

\usepackage{xargs}

\usepackage{amsmath}
\usepackage{amsxtra,amsfonts,amscd,amssymb,bm,bbm,mathrsfs}

\usepackage{amsthm}
\usepackage{mathtools}

\usepackage{enumitem}
\usepackage{multirow}
\usepackage{multicol}
\usepackage{booktabs}
\usepackage[toc,page,header]{appendix}
\usepackage{minitoc}
\usepackage{amssymb,amsmath,amsfonts}
\usepackage{amsthm}
\usepackage{mathtools}
\usepackage{graphicx}
\usepackage{epstopdf}
\usepackage{bm}
\usepackage{bbm}
\usepackage{color}
\usepackage{algorithm}
\usepackage[noend]{algorithmic}
\usepackage{hyperref}
\hypersetup{
  colorlinks,
  breaklinks=true,
  citecolor=blue!70!black,
  linkcolor=blue!70!black,
  urlcolor=blue!70!black
}

\usepackage{multirow}
\usepackage{multicol}
\usepackage{latexsym}
\usepackage{amsmath}
\usepackage{amssymb}
\usepackage{amsthm}
\usepackage{epsfig}
\usepackage{xcolor}
\usepackage{graphicx}
\usepackage{bm}
\usepackage{dsfont}
\usepackage{enumitem}
\usepackage{natbib}
\bibliographystyle{plainnat}

\usepackage[nameinlink,capitalize]{cleveref}

\newcommand{\pref}[1]{\cref{#1}}
\newcommand{\pfref}[1]{Proof of \pref{#1}}

\renewcommand{\eqref}[1]{\texorpdfstring{\hyperref[#1]{Eq. (\ref*{#1})}}{Eq. (\ref*{#1})}}
\crefformat{equation}{#2(#1)#3}
\Crefformat{equation}{#2(#1)#3}

\Crefformat{figure}{#2Figure #1#3}
\Crefname{assumption}{Assumption}{Assumptions}
\Crefformat{assumption}{#2Assumption #1#3}

\usepackage{crossreftools}
\pdfstringdefDisableCommands{%
    \let\Cref\crtCref
    \let\cref\crtcref
}

\newtheorem{theorem}{Theorem}
\newtheorem{lemma}[theorem]{Lemma}
\newtheorem{corollary}[theorem]{Corollary}
\newtheorem{assumption}{Assumption}
\newtheorem{proposition}[theorem]{Proposition}

\newtheorem{definition}{Definition}

\theoremstyle{definition}
\newtheorem{remark}[theorem]{Remark}

\usepackage[algo2e,linesnumbered,vlined,ruled]{algorithm2e}
\usepackage{algorithm}
\usepackage[noend]{algorithmic}

\usepackage{url}
\usepackage{breakcites}
\usepackage{cancel}
\usepackage{enumitem}
\usepackage{comment}
\crefformat{equation}{#2(#1)#3}
\Crefformat{equation}{#2(#1)#3}

\usepackage{booktabs}       %
\usepackage{multirow}
\usepackage{multicol}
\usepackage{makecell}
\usepackage{array}
\newcolumntype{H}{>{\setbox0=\hbox\bgroup}c<{\egroup}@{}}
\newcolumntype{Z}{>{\setbox0=\hbox\bgroup}c<{\egroup}@{\hspace*{-\tabcolsep}}}

\usepackage[normalem]{ulem}
\usepackage{exscale}
\usepackage{tikz}
\usepackage{float}
\usepackage{graphicx,graphics,subfig}

\allowdisplaybreaks

\usepackage{chngcntr}
\usepackage{apptools}

\let\oldparagraph=\paragraph
\renewcommand\paragraph[1]{\oldparagraph{#1.}}

\usepackage{crossreftools}
\pdfstringdefDisableCommands{%
    \let\Cref\crtCref
    \let\cref\crtcref
}

\usepackage{setspace}

\newcommand{\jmlrQED}{\qed}
}

\colt{

\usepackage{hyperref}
\usepackage[capitalize]{cleveref}
\usepackage{aliascnt}

\makeatletter
\@ifundefined{lemma}{}{%

}
\makeatother

\newaliascnt{lemma}{theorem}
\newtheorem{lemma}[lemma]{Lemma}
\aliascntresetthe{lemma}

\crefname{lemma}{Lemma}{Lemmas}
\Crefname{lemma}{Lemma}{Lemmas}

\makeatletter
\@ifundefined{proposition}{}{%

}
\makeatother

\newaliascnt{proposition}{theorem}
\newtheorem{proposition}[proposition]{Proposition}
\aliascntresetthe{proposition}

\crefname{proposition}{Proposition}{Propositions}
\Crefname{proposition}{Proposition}{Propositions}

\makeatletter
\@ifundefined{definition}{}{%

}
\makeatother

\newaliascnt{definition}{theorem}
\newtheorem{definition}[definition]{Definition}
\aliascntresetthe{definition}

\crefname{definition}{Definition}{Definitions}
\Crefname{definition}{Definition}{Definitions}

\makeatletter
\@ifundefined{remark}{}{%

}
\makeatother

\newaliascnt{remark}{theorem}
\newtheorem{remark}[remark]{Remark}
\aliascntresetthe{remark}

\crefname{remark}{Remark}{Remarks}
\Crefname{remark}{Remark}{Remarks}

\makeatletter
\@ifundefined{corollary}{}{%

}
\makeatother

\newaliascnt{corollary}{theorem}
\newtheorem{corollary}[corollary]{Corollary}
\aliascntresetthe{corollary}

\crefname{corollary}{Corollary}{Corollarys}
\Crefname{corollary}{Corollary}{Corollarys}

\renewcommand{\cref}{\Cref}
\addtocontents{toc}{\protect\setcounter{tocdepth}{0}}

  \usepackage{times}

\usepackage[utf8]{inputenc} %
\usepackage[T1]{fontenc}    %

\usepackage{parskip}
\usepackage{xargs}

\usepackage{amsmath}
\usepackage{amsxtra,amsfonts,amscd,amssymb,bm,bbm,mathrsfs}
\usepackage{nicefrac}       %
\usepackage{mathtools}
\usepackage{mathabx}

\usepackage{thmtools}

\newtheorem{assumption}[theorem]{Assumption}

\usepackage{url}
\usepackage[capitalize]{cleveref}
\usepackage[titletoc,title]{appendix}
\usepackage{cancel}
\usepackage{enumerate}
\usepackage{enumitem}
\usepackage{comment}
\crefformat{equation}{#2(#1)#3}
\Crefformat{equation}{#2(#1)#3}
\crefformat{assumption}{Assumption #2#1#3}
\crefformat{conjecture}{Conjecture #2#1#3}

\usepackage{algorithm}

\usepackage{booktabs}       %
\usepackage{multirow}
\usepackage{multicol}
\usepackage{makecell}
\usepackage{array}
\newcolumntype{H}{>{\setbox0=\hbox\bgroup}c<{\egroup}@{}}
\newcolumntype{Z}{>{\setbox0=\hbox\bgroup}c<{\egroup}@{\hspace*{-\tabcolsep}}}

\usepackage[normalem]{ulem}
\usepackage{exscale}
\usepackage{tikz}
\usepackage{float}
\usepackage{graphicx,graphics}
\graphicspath{ {./fig/} }
\allowdisplaybreaks

\let\oldparagraph=\paragraph
\renewcommand\paragraph[1]{\oldparagraph{#1.}}

\newcommand{\pref}[1]{\cref{#1}}
\newcommand{\pfref}[1]{Proof of \pref{#1}}

\renewcommand{\eqref}[1]{\texorpdfstring{\hyperref[#1]{Eq. (\ref*{#1})}}{Eq. (\ref*{#1})}}
\crefformat{equation}{#2(#1)#3}
\Crefformat{equation}{#2(#1)#3}

\Crefformat{figure}{#2Figure #1#3}
\Crefname{assumption}{Assumption}{Assumptions}
\Crefformat{assumption}{#2Assumption #1#3}

\usepackage{crossreftools}
\pdfstringdefDisableCommands{%
    \let\Cref\crtCref
    \let\cref\crtcref
}

\usepackage[noend]{algorithmic}

}
\newcommand{\sups}[1]{^{{\scriptscriptstyle#1}}}
\newcommand{\subs}[1]{_{{\scriptscriptstyle#1}}}

\DeclarePairedDelimiter{\abs}{\lvert}{\rvert} %
\DeclarePairedDelimiter{\brk}{[}{]}
\DeclarePairedDelimiter{\crl}{\{}{\}}
\DeclarePairedDelimiter{\prn}{(}{)}
\DeclarePairedDelimiter{\nrm}{\|}{\|}
\DeclarePairedDelimiter{\tri}{\langle}{\rangle}

\DeclarePairedDelimiter{\ceil}{\lceil}{\rceil}
\DeclarePairedDelimiter{\floor}{\lfloor}{\rfloor}

\DeclareMathOperator{\En}{\mathbb{E}}

\DeclareMathOperator*{\argmin}{arg\,min} %

\newcommand{\wt}[1]{\widetilde{#1}}
\newcommand{\wh}[1]{\widehat{#1}}
\newcommand{\wb}[1]{\widebar{#1}}

\newcommand{\alg}{\Alg}

\newcommand{\R}{\mathbb{R}} %

\newcommand{\hth}{\wh{\theta}}
\newcommand{\lsim}{{\;\raise0.3ex\hbox{$<$\kern-0.75em\raise-1.1ex\hbox{$\sim$}}\;}}
\newcommand{\gsim}{{\;\raise0.3ex\hbox{$>$\kern-0.75em\raise-1.1ex\hbox{$\sim$}}\;}}

\newcommand{\eps}{\varepsilon} 
\newcommand{\RNum}[1]{\uppercase\expandafter{\romannumeral #1\relax}}

\newcommand{\Reg}{\mathbf{Reg}}

\newcommand{\cD}{\mathcal{D}}
\newcommand{\cE}{\mathcal{E}}
\newcommand{\cF}{\mathcal{F}}

\newcommand{\cH}{\mathcal{H}}

\newcommand{\cK}{\mathcal{K}}

\newcommand{\cM}{\mathcal{M}}

\newcommand{\cP}{\mathcal{P}}

\newcommand{\cR}{\mathcal{R}}
\newcommand{\cS}{\mathcal{S}}

\newcommand{\cX}{\mathcal{X}}

\DeclareFontFamily{U}{mathx}{\hyphenchar\font45}
\DeclareFontShape{U}{mathx}{m}{n}{<-> mathx10}{}
\DeclareSymbolFont{mathx}{U}{mathx}{m}{n}
\DeclareMathAccent{\widebar}{0}{mathx}{"73}

\newcommand{\ldef}{\vcentcolon=}
\newcommand{\rdef}{=\vcentcolon}
\newcommand{\Unif}{\mathrm{Unif}}

\newcommand{\spn}{\mathrm{span}}

\def\medskip{\vskip 10 pt}
\def\bigskip{\vskip 15 pt}

\def\texitem#1{\par\vspace{5pt}
\noindent\hangindent 20pt
\hbox to 20pt {\hss #1 ~}\ignorespaces}

\newcommand{\LDPtag}{{\scriptscriptstyle\mathsf{LDP}}}

\newcommand{\SQtag}{{\normalfont\scriptscriptstyle \textsf{-SQ}}}
\newcommandx{\pdecltau}[1][1=\tau]{{\normalfont \textsf{p-dec}}^{#1\SQtag}}

\newcommandx{\Whp}[1][1=\delta]{With probability at least $1-#1$}
\newcommandx{\whp}[1][1=\delta]{with probability at least $1-#1$}

\newcommand{\Mstar}{M^\star}

\usepackage{pifont}

\newcommand{\id}{I}

\newcommand{\RR}{\mathbb{R}}

\newcommand{\EE}{\mathbb{E}}
\newcommand{\PP}{\mathbb{P}}

\newcounter{cnt}
\foreach \num in {1,2,...,26}{%
  \stepcounter{cnt}%
  \expandafter\xdef \csname c\Alph{cnt}\endcsname {\noexpand\mathcal{\Alph{cnt}}}%
  \expandafter\xdef \csname b\Alph{cnt}\endcsname {\noexpand\mathbb{\Alph{cnt}}}%
}

\newcommand{\tr}{\mathrm{tr}}

\DeclarePairedDelimiterX{\ddiv}[2]{(}{)}{%
  #1\;\delimsize\|\;#2%
}

\newcommand{\indic}[1]{\mathbf{1}\left\{#1\right\}} %

\newcommand{\Alg}{\mathsf{Alg}}

\newcommand{\paren}[1]{{\left( #1 \right)}}
\newcommand{\brac}[1]{{\left[ #1 \right]}}

\newcommand{\normal}[1]{\mathsf{N}\paren{#1}}

\newcommandx{\VM}[1][1=M]{V\sups{#1}}

\newcommandx{\fm}[1][1=M]{f\sups{#1}}

\newcommandx{\muM}[1][1=M]{\nu\subs{#1}}

\newcommandx{\PM}[3][1=M,2=\pi,3=\pr]{#3\!\circ\!#1(#2)}

\newcommand{\pr}{\mathsf{Q}}

\newcommand{\gfunc}{g}
\newcommandx{\gm}[1][1=M]{\gfunc\sups{#1}}

\newcommand{\cFp}{\cF^+}

\newcommandx{\risk}[1][1=M]{\mathsf{Risk}^{#1}}

\newcommand{\bpi}{\boldsymbol{\pi}}

\newcommand{\tp}{\sups\top}

\newcommandx{\Mcxt}[1][1=M]{#1_{\sf cxt}}

\newcommand{\ZZ}{\mathbb{Z}}

\newcommandx{\Dl}[1][1=\lf]{\mathsf{D}_{#1}}
\newcommand{\lf}{\ell}

\newcommandx{\DC}[2][1=\Delta]{N_{\mathsf{frac}}(#2,#1)}
\newcommandx{\pds}[1][1=\Delta]{p_{#1}^\star}

\newcommandx{\NM}[2][1=\cM]{N(#1,#2)}

\newcommand{\phat}{\widehat{p}}

\newcommandx{\pim}[1][1=M]{\pi\sups{#1}}
\newcommandx{\pip}[1][1=\cP]{\pi\sups{#1}}
\newcommandx{\pims}{\pim[\Mstar]}
\newcommandx{\Vm}[1][1=M]{V\sups{#1}}
\newcommandx{\Vmm}[1][1=M]{V\sups{#1}(\pi\sups{#1})}

\newcommandx{\LM}[2][1=M]{L(#1,#2)}

\newcommand{\ind}[1]{_{#1}}

\newcommandx{\Enmpi}[3][1=M,2=\pi]{\En\sups{#1,#2}\brac{#3}}
\newcommandx{\Emalg}[3][1=M,2=\alg]{\EE\sups{#1,#2}\brac{#3}}
\newcommandx{\Pmalg}[3][1=M,2=\alg]{\PP\sups{#1,#2}\paren{#3}}

\newcommandx{\bpr}[1][1=\lf]{\pr_{#1}}

\newcommandx{\Mpara}[1][1=M]{\theta(#1)}
\newcommandx{\RISK}[2][1=T,2=\xspace]{\mathfrak{M}_{#1}^{#2}}
\newcommandx{\RISKob}[1][1=T]{\mathfrak{M}_{#1}^{\mathsf{obl}}}
\newcommandx{\SC}[2][1=\Delta,2=\xspace]{\mathfrak{C}_{#1}^{#2}}
\newcommandx{\SCob}[1][1=\Delta]{\mathfrak{C}_{#1}^{\mathsf{ob}}}

\newcommandx{\Ncov}[3][1=\xspace,2=\Delta]{N_{#1}(#3,#2)}

\newcommand{\MPow}{\mathscr{P}}
\newcommandx{\cMPow}[1][1=\MPow]{\cM_{#1}}

\newcommandx{\SQ}[1][1=M]{\mathsf{STAT}_{#1}^{\tau}}
\newcommandx{\VSTAT}[1][1=M]{\mathsf{VSTAT}_{#1}^{\tau}}
\newcommandx{\GSQ}[1][1=M]{\mathsf{GQ}_{#1}^{\tau}}
\newcommandx{\phq}[2][1=\bpi]{#2(#1)}
\newcommandx{\Dph}[2][1={\phi}]{\mathsf{D}_{#1}\paren{#2}}

\newcommandx{\hgm}[3][1=M,2=\delta]{\widehat{L}_{#2}(#1,#3)}

\newcommandx{\Tdec}[2][1=\Delta]{\mathfrak{C}^{\,\sf dec}_{#1}(#2)}

\newcommand{\SQdim}{\mathsf{SQDim}}
\newcommandx{\SQD}[3][1=\beta,2=\tau]{\SQdim^{#2}_{#1}(#3)}

\newcommandx{\DCF}[3][1=\Delta,2={\cF},3={\cFp}]{\DC[#1]{#2,#3}}
\newcommandx{\SCDP}[1][1=\Delta]{\SC[#1][\LDPtag]}

\newcommandx{\IDC}[1]{N_{\mathsf{frac}}(#1)}

\newcommand{\clip}{\mathsf{clip}}

\usepackage[suppress]{color-edits}
 \addauthor{sr}{red}
 \addauthor{fc}{blue}
 \addauthor{jq}{brown}

\usepackage{mathtools}  

\arxiv{
\title{Self-Normalized Martingales and Uniform Regret Bounds for Linear Regression}
\author{
  Fan Chen \\ {\small MIT} \\ {\small \texttt{fanchen@mit.edu}} 
  \and Jian Qian \\ {\small University of Hong Kong} \\ {\small \texttt{jianqian@hku.hk}} 
  \and Alexander Rakhlin \\ {\small MIT} \\ {\small \texttt{rakhlin@mit.edu}}
    \and Nikita Zhivotovskiy \\ {\small UC Berkeley} \\ {\small \texttt{zhivotovskiy@berkeley.edu}}
}
}
\colt{
\title[Self-normalized martingales under smoothness assumption]{Self-Normalized Martingales and Uniform Regret Bounds for Linear Regression}

\coltauthor{%
 \Name{Author Name1} \Email{abc@sample.com}\\
 \addr Address 1
 \AND
 \Name{Author Name2} \Email{xyz@sample.com}\\
 \addr Address 2%
}
}

\begin{document}

\maketitle

\begin{abstract}
Self-normalized martingale inequalities lie at the heart of confidence ellipsoids for online least squares and, more broadly, many bandit and reinforcement-learning results. Yet existing vector and scalar results typically rely on bounded covariates and an explicit regularization matrix, producing bounds that are \emph{not scale-invariant}: although the self-normalized quantity is scale-invariant by definition, its standard upper bounds are not.

We characterize when scale-invariant upper bounds on self-normalized martingales are possible. Without further assumptions, we prove that nontrivial scale-invariant bounds exist only in dimension $d=1$; moreover, in $d=1$ we obtain $O(\log T)$ scale-invariant self-normalized bounds without any assumptions on the covariates. In contrast, for $d>1$ we show that no nontrivial scale-invariant bound can hold in full generality. We then connect this dichotomy to \emph{doubly-uniform} regret in online linear regression (i.e., regret bounds that are simultaneously independent of the covariate scale and the comparator norm) and use it to resolve the open question of Gaillard, Gerchinovitz, Huard, and Stoltz, \emph{``Uniform regret bounds over $\mathbb{R}^d$ for the sequential linear regression problem with the square loss''} (ALT 2019): in $d=1$ we give an explicit algorithm with $O(\log T)$ doubly-uniform regret, whereas for $d>1$ sublinear doubly-uniform regret is impossible.

Finally, under a natural \emph{smoothness} condition (bounded Radon--Nikodym derivatives of the conditional covariate laws with respect to a fixed base measure), we recover sublinear regret for $d>1$ without bounded covariates and derive a self-normalized concentration inequality free of the usual regularization penalties, yielding arguably a first natural scale-invariant bound for adaptive, non-i.i.d.\ vector martingales.
\end{abstract}

\section{Introduction}

Our play opens with what appears to be a two-protagonist drama. They make their separate entrances across two acts, each insisting on a solo turn in the spotlight. In the final act they meet, only to realize that they are, in fact, a mirror image of each other.
\paragraph{Act I: Self-normalized martingales} In an introductory course on Probability, we learn that the sum of normal random variables is also normal, and, in particular, 
$\frac{\sum_{i=1}^T X_i}{\sqrt{\sum_{i=1}^T \sigma_i^2}} \sim \normal{0,1}$
where $X_i\sim \normal{0, \sigma_i^2}$ are independent normal random variables and $\sigma_i$ are  constants. If one replaces the denominator by the random realizations of $X_i$, then the distribution of $\frac{\sum_{i=1}^T X_i}{\sqrt{\sum_{i=1}^T X_i^2}}$ is no longer normal, yet its tails are similar to those of the normal distribution. In particular, an application of Hoeffding's inequality (via symmetrization) shows that  
$$\PP\left(\frac{\sum_{i=1}^T X_i}{\sqrt{\sum_{i=1}^T X_i^2}} > u\right) \leq \exp\left(-\frac{u^2}{2}\right)$$ for any \emph{symmetric} independent random variables $X_i$, remarkably under no further assumptions on their distributions (see e.g., \citet[Ex. 7.3]{van2014probability}). To bring out the symmetry requirement, we may instead write the above inequality as
\begin{align}
\PP\left(\frac{\sum_{i=1}^T \varepsilon_i X_i}{\sqrt{\sum_{i=1}^T X_i^2}} > u\right) \leq \exp\left(-\frac{u^2}{2}\right),
\label{eq:hoeffding-selfnorm}
\end{align}
where $\varepsilon_i$ are i.i.d. Rademacher random variables, and $X_i$ are independent with no further assumption on their distributions. Such inequalities for \emph{self-normalized sums} are very attractive, both due to the lack of assumptions and due to their natural scale invariance. 

It is then reasonable to ask whether inequalities similar to \eqref{eq:hoeffding-selfnorm} hold for martingales. More precisely, suppose for simplicity that $X_i$ are deterministic functions of $\varepsilon_1,\ldots,\varepsilon_{i-1}$, i.e. measurable with respect to the {dyadic} filtration. Clearly, the sum $\sum_{i=1}^T \varepsilon_i X_i$ is a martingale. Does the inequality \eqref{eq:hoeffding-selfnorm} also hold for this martingale without assumptions on $X_i$'s?

Perhaps surprisingly, the answer is no. Even more interestingly, this answer is related to what is referred to as ``doubly uniform regret'' in online learning. But we are getting ahead of ourselves.

For simplicity of exposition, let us consider the first moment of the ratio, rather than the tail bound. The following inequality can be found in the classical book of \citet[p. 199]{depena2009self}:
\begin{align}
\En\left[S_T/V_T^{1/2}\right] \leq C+ c\En\brk*{0\vee\log \log (V_T^{1/2}\vee V_T^{-1/2})},
\label{eq:depena-selfnorm}
\end{align}
where, henceforth, we abbreviate $S_T=\sum_{i=1}^T \varepsilon_i X_i$ and $V_T=\sum_{i=1}^T X_i^2$, and $C,c>0$ are constants. In the multi-dimensional setting, described below, the upper bound also involves the logarithm of the condition number of the matrix $V_T$ (see additionally \citet[Corollary 4.5]{whitehouse2023time} and references therein).

The inequality \eqref{eq:depena-selfnorm} lacks the scale-invariance property that we desire: the left-hand side does not change when multiplying all $X_i$'s by a constant, yet the right-hand side does. This lack of scale-invariance is present in many results in the literature on self-normalized martingales, and it stems from the pseudo-maximization (or the method of mixtures) technique pioneered by \citet{robbins1970boundary} and used extensively by \citet{depena2004self,depena2009self}. The method aims to place a non-trivial mass on a parameter that can only be known after observing the scale of the realization. 

One approach to remove scale-dependence in the upper bound on the self-normalized martingale is to change the denominator. This can be achieved by augmenting $V_T$ with a regularization term. For instance, \citet{depena2004self} establishes
\begin{align}
    \PP\left(\frac{S_T}{(V_T+\En V_T)^{1/2}} > u\right) \leq \sqrt{2}\exp\left(-\frac{u^2}{4}\right).
    \label{eq:selfnorm-by-mixture}
\end{align}
The expected value $\En V_T$ can be viewed as fixing the scale of the problem, yet its presence in the denominator is not desirable. Another approach is to augment $V_T$ with a constant, making both the ratio and the ensuing upper bound scale-dependent. 

In particular, the addition of a regularizing constant to $V_T$ has been employed in the multi-dimensional setting (that is, $X_i=X_i(\varepsilon_1,\ldots,\varepsilon_{i-1})$ are taking values in $\RR^d$ and $V_T=\sum_{i=1}^T X_i X_i^\top$ is the sample covariance) by \citet[Theorem 14.7]{depena2009self} and \cite{abbasi2011improved} to establish tail bounds of the form
\begin{align}
\label{eq:self-norm-with-Gamma}
S_T^\top (V_T+\Gamma)^{-1}S_T \lesssim
\log \left(\frac{\det(V_T+\Gamma)}{\det(\Gamma)}\right)+\log(1/\delta)
\end{align}
with probability at least $1-\delta$, for some deterministic positive definite matrix $\Gamma$. Once again, both sides are not scale-invariant, limiting the applicability of the bound when the scale is unknown.

Before continuing our discussion, we mention that the analysis of self-normalized martingales plays a central role in bandits and reinforcement learning: it underlies confidence ellipsoids for online least-squares estimators (e.g., \cite{dani2008stochastic,abbasi2011improved}) and the resulting online-to-confidence set conversions (see \cite{abbasi2012online,lee2024improved,clerico2025confidence} and the textbook \citep[Chapter~20]{lattimore2020bandit} for bibliographic pointers), identification in Linear Time-Invariant systems \citep{simchowitz18a,sarkar2018nearoptimalfinitetime}, and beyond. More broadly, there is renewed interest in formulations and in weakening tail assumptions beyond the classical conditionally sub-Gaussian setting (e.g., \cite{howard2020time,whitehouse2023time,zhao2023variance,ziemann2024vector,akhavan2025bernstein}).

\paragraph{Act II: Doubly uniform regret}

One of the first online prediction methods is the celebrated Vovk-Azoury-Warmuth (VAW) estimator, initially proposed by \citet{vovk1997competitive} and later refined by \citet{v-cos-01, azoury2001relative} (see also \citep[Theorem 11.9]{cesa2006prediction}). First, we recall that in online supervised learning with squared loss, on each round $t\in[T]$, the forecaster observes $x_t\in\RR^d$, selects a prediction $\yhat_t\in\RR$ and observes $y_t\in[-1,1]$. The VAW estimator, defined later in the text, is essentially a regularized least squares estimator with a regularization parameter $\lambda>0$, and it achieves the following regret bound: for any $\theta\in\RR^d$,
\begin{align}
    \label{eq:vaw-regret}
    \sumt[T] (\yhat\ind{t}-y\ind{t})^2- \sumt[T] (\tri{\theta, x\ind{t}}-y\ind{t})^2
    \leq \lambda\nrm{\theta}^2 + d\log\prn*{1+\frac{T\max_{t} \nrm{x\ind{t}}^2}{\lambda}}.
\end{align}
 Notably, the bound is non-uniform in two ways: it depends both on the norm of the comparator vector $\theta$ and the scale of the covariates. \cite{pmlr-v40-Bartlett15} raised the question of whether one can obtain a regret bound that is uniform over all $\theta$, as the existing lower bounds do not show this necessity (see, e.g., \cite[Theorem 4]{gaillard2019uniform}). This led the authors of \cite{gaillard2019uniform} to further ask whether regret bounds that are doubly uniform---with respect to the norm of the target parameter and the covariates---are possible in online linear regression.

So far, this double uniformity is only known in the so-called transductive online setup, where all design vectors are arbitrary but known in advance so that the predictor can use them \citep{pmlr-v40-Bartlett15, gaillard2019uniform, qian2024refined}, and, roughly speaking, establish the scale of the prediction problem. This double uniformity also appears in the statistical (i.i.d.) setup, where variants of non-linear predictors allow one to bypass the dependence on both the distribution of the design and the norm of the parameter \citep{forster2002relative, mourtada2022distribution}. More generally, in the context of GLMs there is recent interest in analyzing unbounded parameter spaces, for example in classification with logistic regression, where large parameter norms are very natural and relate to (almost) linearly separable samples. In the transductive setup and in the context of online logistic regression, see \cite{drmota2026phase}, while \cite{qian2024refined} focuses on regression with square, hinge and logarithmic losses.

\paragraph{Act III: On the equivalence of martingale bounds and regret inequalities}

Denote the regret of the learner in the online prediction problem with arbitrary $\theta\in\RR^d$ as
\begin{align}
    \label{eq:reg_def}
    \Reg(T)\ldef \sumt (\yhat\ind{t}-y\ind{t})^2-\inf_{\theta\in\RR^d} \sumt (\tri{\theta, x\ind{t}}-y\ind{t})^2.
\end{align}
Suppose the forecaster attempts to predict the following sequence. Covariates form a predictable process $x_t=X_t(\varepsilon_1,\ldots,\varepsilon_{t-1})$, as earlier in the text, and the outcome variable $y_t=\varepsilon_t $ is an independent Rademacher random variable. It is clear that in this setting, the best strategy for the forecaster is to predict $\yhat_t=0$ for all $t$. Then, the expected regret of the forecaster is given by
\begin{align}
    \En_{\varepsilon}\brk*{\sumt y_t^2-\inf_{\theta\in\RR^d} \sumt (\tri{\theta, x\ind{t}}-y\ind{t})^2 }&=\En_{\varepsilon}\sup_{\theta\in\RR^d}  2\tri{\theta, \sumt \varepsilon_t X_t} - \tri{\theta, \sumt X_t X_t^\top \theta} \\
    &= \En_{\varepsilon}\brk*{S_T^\top (V_T)^{\dagger} S_T}.
\end{align}
which is precisely the expected value of the self-normalized process that appeared in Act I. Furthermore, \cite{rakhlin2014nonparametric} proved a converse statement (for a more general setting of regression with any class of functions): no matter what the sequence of $\{(x_t,y_t)\}_{t\in[T]}$ is, even if chosen adaptively by Nature, there \emph{exists} a prediction strategy that achieves a regret bound that is, up to a multiplicative constant, in the above self-normalized form for the worst-case martingale (see below for more details). %
Thus, upper bounds on \eqref{eq:reg_def} for all sequences imply upper bounds for the expected self-normalized ratio, and vice versa.

The connection  between probabilistic martingale inequalities and regret bounds has been a focus of extensive research, including \cite{abernethy2008optimal,rakhlin2010online,rakhlin2017equivalence, foster2018online, foster2017zigzag, beiglbock2015pathwise, orabona2023tight}, and, in particular, certain \textit{equivalence} between these two seemingly unrelated fields was studied in \cite{rakhlin2017equivalence,foster2018online}.

\paragraph{Finale}

Due to the two-sided equivalence between minimax regret bounds for unbounded comparators $\theta\in\RR^d$ and expected value of self-normalized martingales, we can establish lower/upper bounds for one by studying the other, whichever is more convenient. In particular, the issues discussed in Act I regarding the knowledge of the scale of the problem are precisely the issues discussed in Act II regarding the knowledge of the norm of the comparator vector $\theta$ and the scale of the covariates. In particular, later in the paper, we describe the exact link between self-normalized bounds of the form \eqref{eq:self-norm-with-Gamma} and the Vovk-Azoury-Warmuth forecaster. 

In particular, our contributions are:
\begin{itemize}
    \item We establish a sharp separation between the cases $d=1$ and $d>1$.
    When $d=1$, we prove a \emph{fully scale-invariant} bound of order $O(\log T)$ for self-normalized martingales \emph{without any assumption on the covariates}.
    Via the regret--martingale connection developed in \cite{rakhlin2017equivalence}, this implies a doubly-uniform $O(\log T)$ regret guarantee for online linear regression, thereby resolving the question of \cite{gaillard2019uniform} in dimension one.
    Moreover, we provide an explicit algorithm achieving this doubly-uniform $O(\log T)$ regret.

    \item In contrast, when $d>1$ we show that no nontrivial scale-invariant control of self-normalized vector martingales is possible in full generality, and consequently sublinear doubly-uniform regret bounds for online linear regression cannot hold.
    This completes our answer to the question of \cite{gaillard2019uniform}. %

    \item On the positive side, still in the regime $d>1$, we introduce a \emph{smoothness} condition on the covariate process, requiring that each conditional law admits a bounded Radon--Nikodym derivative with respect to a fixed base measure.
    Under this assumption we obtain sublinear regret \emph{without} assuming bounded covariates.
    Moreover, our bounds avoid the usual matrix regularization penalties (e.g., the log-determinant term in \eqref{eq:self-norm-with-Gamma}), yielding what appears to be a first natural example of a scale-invariant self-normalized martingale bound in a genuinely non-i.i.d.\ setting.
\end{itemize}

\section{Preliminaries}
\label{sec:prlim}

In the previous section, we motivated the study of {dyadic} self-normalized martingales of the form
$\varepsilon_t X_t$, where $(\varepsilon_t)_{t\ge1}$ are i.i.d.\ Rademacher signs and each $X_t$ is
$\sigma(\varepsilon_1,\ldots,\varepsilon_{t-1})$-measurable. In particular, if $(X_t)$ is deterministic
(or more generally independent of $(\varepsilon_t)$), then conditioning on $(X_t)_{t\le T}$ and applying
Hoeffding's inequality yields the scale-invariant tail bound \eqref{eq:hoeffding-selfnorm} for the ratio
$\sum_{t=1}^T \varepsilon_t X_t/\sqrt{\sum_{t=1}^T X_t^2}$.
Let us now present the more general filtered definition that is standard in the online regression and
bandit literature (e.g., \citealp{abbasi2011improved,ziemann2024vector}), and that will serve as our
main probabilistic object. Throughout this probabilistic discussion we use capitals $(X_t,Y_t)$; later,
when we switch to the online learning protocol, we will revert to the conventional lowercase notation
$(x_t,y_t)$ and use a separate notion of game history that also records predictions.

Let $(\Omega,\mathcal F,\PP)$ be a probability space equipped with a filtration
$(\mathcal G_t)_{t\ge0}$. 
We assume that $X_t$ is \emph{predictable} and that $(Y_t)_{t=1}^T$ is a real-valued martingale difference sequence with
respect to $(\mathcal G_t)$:
\begin{equation}
    X_t\in\RR^d ~\text{is }~ \mathcal G_{t-1}~\text{-measurable and}\qquad \En\left[Y_t \,\middle| \mathcal G_{t-1}\right]=0
    \qquad\text{for all } t\in[T].\label{eq:def_martingale_XY}
\end{equation}
Depending on the application, one may further assume boundedness
$|Y_t|\le 1$ almost surely or a conditional sub-Gaussian condition. Define the cumulative vector and the Gram matrix
\[
S_t \ldef \sum_{i=1}^t Y_i X_i  \in \RR^d,
\qquad
\cov_t \ldef \sum_{i=1}^t X_i X_i^\top \in \RR^{d\times d}.
\]
The canonical scale-free quantity is the \emph{self-normalized} process
\[
R_t \ldef \|S_t\|_{\cov_t^\dagger}^2 = S_t^\top \cov_t^\dagger S_t,
\]
where $\cov_t^\dagger$ denotes the Moore--Penrose pseudoinverse. Controlling $R_t$ (in expectation or
with high probability) is a central theme of the self-normalization literature
\citep{de1999general,depena2009self,bercu2019new} and, in the online learning context,
it is the quantity that underlies confidence ellipsoids and regret bounds in least squares and linear
bandits \citep{abbasi2011improved, ziemann2024vector}.

The question we pursue is whether martingale analogues of \eqref{eq:hoeffding-selfnorm} can hold for
$R_t$ under minimal assumptions on the predictable covariates $(X_t)$. Since we aim for a uniform upper
bound that holds for all martingales of the form \cref{eq:def_martingale_XY}, we define
\begin{equation}
\label{eq:cR_d}
\cR_d(T)= \sup_{P_{X,Y}} \En_{P_{X,Y}}\brk*{\nrm{S_T}_{\cov_T^\dagger}^2},
\end{equation}
where the supremum ranges over all laws $P_{X,Y}$ of dimension $d$ satisfying
\cref{eq:def_martingale_XY} and such that $|Y_t|\le 1$ almost surely for all $t\in [T]$.

A \emph{dyadic martingale} is a special case where $Y_t=\varepsilon_t$ are i.i.d.\ Rademacher and
$X_t = X_t(\varepsilon_1,\ldots,\varepsilon_{t-1})$ is a deterministic function of the past signs.
Such a process can be viewed as an $\RR^d$-valued \emph{tree} $X$ of depth $T$, or a sequence of
mappings $X_t\colon \{\pm1\}^{t-1}\to\RR^d$. Let
\[
\cR_d^{\mathrm{dyadic}}(T)\ldef \sup_X \En_\varepsilon\brk*{R_T},
\]
where the supremum is over all trees $X$ and expectation is over i.i.d.\ Rademacher $(\varepsilon_t)$.

\begin{lemma}\label{lem:dyadic-equivalence}
For any $d\geq 1$ and $T\geq 1$, if we only consider processes where $|Y_t|\leq 1$ a.s., then
\[
\cR_d(T)=\cR_d^{\mathrm{dyadic}}(T).
\]
\end{lemma}
In words, worst-case martingales---from the point of view of self-normalized ratios---are the dyadic
martingales, up to a factor $2$. Since each dyadic martingale is defined by $2^T-1$ values (the number
of nodes in the binary tree), for each such dyadic martingale we can consider its rescaled (by the
maximum norm) variant $X'$ with $\|X'_t\|\leq 1$. Since the value of the self-normalized ratio does
not change when scaled by a constant, we also have the following conclusion. Let
$\cR_d^{\mathrm{bdd}}(T)$ denote the supremum over the processes of the form \cref{eq:def_martingale_XY}
with $\|X_t\|\leq 1$ almost surely, and let $\cR_d^{\mathrm{bdd,dyadic}}(T)$ denote the corresponding
supremum restricted to dyadic and bounded martingales.
\begin{corollary}
For any $d\geq 1$ and $T\geq 1$, if we only consider processes where $|Y_t|\leq 1$ a.s., then
$\cR_d^{\mathrm{dyadic}}(T) = \cR_d^{\mathrm{bdd,dyadic}}(T)$, and thus
\[
\cR_d^{\mathrm{dyadic}}(T)
= \cR_d^{\mathrm{bdd,dyadic}}(T)
= \cR_d^{\mathrm{bdd}}(T)
= \cR_d(T).
\]
\end{corollary}
In the following section, this result will imply that the difficulty in doubly-uniform regret bounds is
a consequence of unbounded $\theta$ rather than unbounded covariates. This is also reflected by our
lower bounds, which hold for bounded covariates.

\subsection{An online prediction game}
\label{subsec:online-game}

We now use the standard online learning notation and write covariates, outcomes, and predictions as
$(x_t,y_t,\yhat_t)$. On each round $t\in[T]$, the environment $\Env$ reveals a covariate vector
$x\ind{t}\in\RR^d$, the learner $\Alg$ outputs a prediction $\yhat\ind{t}\in\RR$, and then $\Env$
reveals an outcome $y\ind{t}\in[-1,1]$. Both $\Env$ and $\Alg$ may be adaptive. We denote the history
prior to round $t$ by
\[
\cH^{t-1}\ldef \sigma\Big(\{(x\ind{s},\yhat\ind{s},y\ind{s})\}_{s<t}\Big).
\]

Given a comparator set $\Theta\subseteq\RR^d$, the square-loss regret is
\begin{equation}
\label{eq:regret-theta}
\Reg_\Theta(T)
\ldef
\sumt (\yhat\ind{t}-y\ind{t})^2
-\inf_{\theta\in\Theta}\sumt \bigl(\tri{\theta,x\ind{t}}-y\ind{t}\bigr)^2.
\end{equation}
Since our focus is on $\Theta=\RR^d$, we abbreviate $\Reg(T)\ldef \Reg_{\RR^d}(T)$.

A convenient way to relate regret to self-normalization is to consider a stochastic environment with
conditionally unbiased outcomes. Fix any algorithm $\Alg$ and a sequential law $P_{x,y}$ over
$(x\ind{1},y\ind{1},\ldots,x\ind{T},y\ind{T})$ such that, when the environment is generated from
$P_{x,y}$ independently of the predictions of $\Alg$, the outcomes satisfy
\begin{equation}
\label{eq:online-unbiased}
\En\left[y\ind{t}\,\middle|\,x\ind{1},y\ind{1},\ldots,x\ind{t-1},y\ind{t-1},x\ind{t}\right]=0
\qquad\text{for all } t\in[T].
\end{equation}
Let $\Env$ be the environment that samples $(x\ind{1},y\ind{1},\ldots,x\ind{T},y\ind{T})\sim P_{x,y}$
and reveals it round by round. Then
\begin{align}
\En\sups{\Env,\Alg}\brk*{\sumt (\yhat\ind{t}-y\ind{t})^2}
&=
\En\sups{\Env,\Alg}\brk*{\sumt \bigl(\yhat\ind{t}^2-2\yhat\ind{t}y\ind{t}+y\ind{t}^2\bigr)} \nonumber\\
&=
\En\sups{\Env,\Alg}\brk*{\sumt \bigl(\yhat\ind{t}^2+y\ind{t}^2\bigr)}
\ge
\En_{P_{x,y}}\brk*{\sumt y\ind{t}^2},
\label{eq:loss-lb-y2}
\end{align}
where the middle equality uses \eqref{eq:online-unbiased} (hence $\En[\yhat\ind{t}y\ind{t}]=0$). Next,
define $S_T \ldef \sumt y\ind{t}x\ind{t}\in\RR^d$ and
$V_T \ldef \sumt x\ind{t}x\ind{t}\tp\in\RR^{d\times d}$. A direct completion of squares gives the
exact identity
\begin{align}
\sumt y\ind{t}^2-\inf_{\theta\in\RR^d}\sumt \bigl(\tri{\theta,x\ind{t}}-y\ind{t}\bigr)^2
=
\sup_{\theta\in\RR^d}\Bigl\{2\tri{\theta,S_T}-\nrm{\theta}_{V_T}^2\Bigr\}
=
\nrm{S_T}_{V_T^\dagger}^2.
\label{eq:selfnorm-identity-game}
\end{align}
Combining \eqref{eq:loss-lb-y2} with \eqref{eq:selfnorm-identity-game} yields
\begin{align}
\En\sups{\Env,\Alg}\brk*{\Reg(T)}
\ge
\En_{P_{x,y}}\brk*{\nrm{S_T}_{V_T^\dagger}^2}.
\label{eq:regret-lb-selfnorm}
\end{align}
This implies that for every algorithm $\Alg$,
\[
\max_{\Env}\En\sups{\Env,\Alg}\brk*{\Reg(T)} \ge \cR_d(T).
\]
Conversely, the work of \citet{rakhlin2017equivalence} provides a minimax upper bound turning
self-normalized control into regret guarantees (up to universal constants), yielding a two-sided link
between self-normalization and optimal regret in online linear regression. We state here the upper
bound of \cite[Lemma~4]{rakhlin2014nonparametric} for the linear function class:
\begin{lemma}
\label{lem:rakhlinsridharan}
In the notation above, it holds that
\[
\min_{\Alg}\max_{\Env}\En\sups{\Env,\Alg}\brk*{\Reg(T)}  \jqedit{\leq 4\cR_{d+1}^{\mathrm{dyadic}}(T)} \le 4\cR_{d+1}(T).
\]
\end{lemma}
We remark that the actual upper bound in the proof is smaller than $\cR_{d+1}(T)$; for our purposes,
this is only important for $d=1$, which we treat separately.

\section{Self-normalized martingales and online prediction in the fully adversarial setting}
\label{sec:full-adversary}

In this section, we first discuss the one-dimensional case $d=1$, and then the case $d\ge 2$.
Remarkably, both regret and the self-normalized martingale exhibit very different behavior in these two
regimes.

\subsection{Dimension one: logarithmic self-normalized martingale and doubly-uniform regret}
\label{sec:1d}

In this section, we focus on the one-dimensional case $d=1$.
We show that (i) the dyadic self-normalized martingale admits $O(\log T)$ control in expectation, and
(ii) the minimax doubly-uniform regret in online linear regression is also $\Theta(\log T)$. Taken
together, these results settle the behavior of both objects of interest in dimension one.

We start from the probabilistic side by establishing a linear bound on the exponential moment of the
one-dimensional dyadic self-normalized martingale.

\begin{theorem}
\label{thm:selfnorm-exp}
For any dyadic martingale $X_1,\ldots,X_T$ and any $c\in(0,1/4]$,
\[
\En_{\varepsilon}\bigl[\exp(c R_T)\bigr]
\le
T\exp\Bigl(\frac{c}{1-2c}\Bigr).
\]
Consequently,
\[
\En_{\varepsilon}[R_T]
\le
\frac{1}{c}\log\left(T\exp\Bigl(\frac{c}{1-2c}\Bigr)\right)
=
\frac{\log T}{c}+\frac{1}{1-2c}.
\]
\end{theorem}

Theorem~\ref{thm:selfnorm-exp} provides a homogeneous, scale-invariant control of the self-normalized
martingale in dimension one. In particular, it improves upon the scale-sensitive behavior suggested by
classical mixture-based bounds such as \eqref{eq:depena-selfnorm}: the right-hand side grows only
logarithmically with $T$ and requires no boundedness or moment assumptions on the predictable covariates
$(X_t)$ beyond measurability with respect to the dyadic filtration.

In light of the regret--martingale connection discussed above, the logarithmic behavior in
Theorem~\ref{thm:selfnorm-exp} suggests that $\log T$ is the correct scale for doubly-uniform regret in
one dimension. We make this precise by giving a matching upper bound via an explicit procedure.

\begin{theorem}
\label{thm:1d-reg-upper}
Suppose that $d=1$ and $|y_t|\le m$ almost surely.
Then there exists an algorithm (\cref{alg:meta}) that achieves deterministically
$\Reg(T)\lesssim m^2\log T$.
\end{theorem}

Complementarily, the minimax lower bound of order $\Omega(\log T)$ follows from
\citet[Theorem~7]{gaillard2019uniform} (adapted from \citet[Theorem~2]{v-cos-01}).

\begin{proposition}
\label{prop:lower-bound-1d}
In the setup of \Cref{thm:1d-reg-upper} with $T\ge 10$ and $m=1$, there exists a dyadic martingale such
that $\En[R_T]\gtrsim \log T$.
\end{proposition}

Together, Theorem~\ref{thm:1d-reg-upper} and Proposition~\ref{prop:lower-bound-1d} yield the claimed
$\Theta(\log T)$ characterization of doubly-uniform regret in dimension one, aligning with the
logarithmic self-normalized control in Theorem~\ref{thm:selfnorm-exp}.

\subsection{Dimension two and higher: linear lower bounds}
\label{sec:lower}

We now turn to $d\ge 2$.
In sharp contrast to the one-dimensional case, the self-normalized \emph{vector} martingale can grow
linearly in the worst case.
Through \eqref{eq:regret-lb-selfnorm}, this implies that doubly-uniform regret cannot be sublinear under
a fully adversarial environment.

\begin{theorem}
\label{prop:lower}
Let $T\ge 1$ be any integer. When $d\ge 2$, for any $\eps\in(0,1)$, there exists a dyadic martingale
$X_1,\ldots,X_T$ such that
\[
\En\left[\nrm{S\ind{T}}_{\covt[T]^\dagger}^2\right]\ge (1-\eps^2)T.
\]
In particular, by \eqref{eq:regret-lb-selfnorm}, there exists an environment such that for any algorithm,
\[
\En\brk*{\Reg(T)}
\ge
\En\left[\nrm{S\ind{T}}_{\covt[T]^\dagger}^2\right]
\ge
(1-\eps^2)T.
\]
\end{theorem}

\begin{proof}[Proof sketch]
We construct an adaptive dyadic martingale that injects a constant amount of self-normalized energy at
every step. Let $(\varepsilon_t)_{t\ge1}$ be i.i.d.\ Rademacher variables, let
$S_t=\sum_{i\le t}\varepsilon_i X_i$, and define the regularized matrix
\[
\lcovt \ldef \sum_{i\le t} X_i X_i^\top + \lambda I,
\]
for some fixed $\lambda>0$. Using the Sherman--Morrison formula and conditioning on the past, one checks
that
\[
\En\left[\|S_{t+1}\|^2_{\lcovt[t+1]^{-1}}\mid\cF_t\right]
=
\|S_t\|^2_{\lcovt^{-1}}
+
\frac{\|X_{t+1}\|^2_{\lcovt^{-1}}
-\langle X_{t+1},\lcovt^{-1}S_t\rangle^2}
{1+\|X_{t+1}\|^2_{\lcovt^{-1}}}.
\]
We choose $X_{t+1}=r\,\lcovt^{1/2}e_t$, where $e_t$ is any unit vector orthogonal to
$\lcovt^{-1/2}S_t$; this is always possible for $d\ge2$.
This choice kills the cross term and ensures $\|X_{t+1}\|^2_{\lcovt^{-1}}=r^2$, so the conditional
increment is the constant $r^2/(1+r^2)$.
Iterating yields $\En[\|S_T\|^2_{\lcovt[T]^{-1}}]=Tr^2/(1+r^2)$.
Moreover, $S_T\in\mathrm{range}(\covt[T])$, so $\nrm{S_T}_{\covt[T]^\dagger}^2\ge \nrm{S_T}_{\lcovt[T]^{-1}}^2$,
and hence $
\En\left[\nrm{S_T}_{\covt[T]^\dagger}^2\right]\ge \frac{Tr^2}{1+r^2}.
$
Taking $r^2=(1-\eps^2)/\eps^2$ completes the lower bound (for any fixed $\lambda>0$).
\end{proof}

Note that since the ratio is homogeneous, the construction in the proof can be rescaled so that
$\|X_t\|\le 1$ deterministically.

\section{Online prediction with smooth environment}
\label{sec:upper}

Even though the worst-case lower bound looks grim, the picture brightens if the environment is forced to
``add randomness''. We formalize this with a \emph{smoothness} condition that caps how concentrated each
$x_t$ can be relative to a fixed base measure $\mu$ \citep{block2022smoothed}.

\begin{assumption}[Smoothness]\label{asmp:smooth}
There exists a probability measure $\mu$ on $\RR^d$ and a parameter $\Ccov\ge 1$ such that for every
round $t$, given the partial history $\cH_{t-1}^x=\crl{x_1,\ldots,x_{t-1}}$, the conditional law of
$x\ind{t}$ is $P_t(\cdot\mid \cH_{t-1}^x)$ and satisfies
\[
\frac{\mathrm d P_t(\cdot\mid \cH_{t-1}^x)}{\mathrm d\mu}(x)\le \Ccov
\qquad\text{for $\mu$-a.e.\ }x\in\RR^d,
\]
for every history $\cH_{t-1}^x$.
\footnote{Here we assume smoothness with respect to the partial history $\cH^x_{t-1}$. While it may be
more natural to assume smoothness given the full history
$\cH_{t-1}=\crl{(x_s,y_s,\yhat_s)}_{s<t}$, requiring smoothness conditional on the partial history is
weaker.}
\end{assumption}

This assumption limits how concentrated the conditional law of $x_t$ can be; for instance, when $\mu$
is non-atomic it rules out choosing $x_t$ deterministically, as in the lower-bound construction of
\cref{prop:lower}. The condition is trivially satisfied when $\cX$ is finite (with
$\mu=\Unif(\cX)$ and $\Ccov\le |\cX|$). When $\cX$ is infinite, however, it can be a fairly strong
assumption on the environment, and can make online prediction significantly easier.

To provide more intuition, \citet{haghtalab2024smoothed} show that under \cref{asmp:smooth},
$x_1,\ldots,x_T$ can be coupled with a subsequence of i.i.d.\ random vectors
(\cref{lem:smooth-to-iid}), echoing the ``smoothed analysis'' philosophy in algorithms.

\subsection{The Vovk--Azoury--Warmuth (VAW) algorithm under the smoothness assumption}

For upper bounds under \cref{asmp:smooth}, we use the Vovk--Azoury--Warmuth (VAW) predictor, a classic
forecaster for online linear regression. The regularized version is well known, but since both the
comparator and the covariates may be unbounded here, standard analyses that rely on a fixed regularizer
do not account for smoothness. We therefore study the \emph{unregularized} VAW (\cref{alg:VAW}) and
tailor the analysis to the smooth setting.

\begin{algorithm}[h]
\begin{algorithmic}[1]
\FOR{$t=1,2,\cdots,T$}
\STATE Set
$
\hth\ind{t}\in \argmin_{\theta\in\RR^d}\ \tri{\theta,x\ind{t}}^2+\sum_{i<t}\bigl(\tri{\theta,x\ind{i}}-y\ind{i}\bigr)^2.
$
\STATE Predict $\yhat\ind{t}=\tri{\hth\ind{t},x\ind{t}}$.
\ENDFOR
\end{algorithmic}
\caption{Vovk--Azoury--Warmuth (VAW) predictor}
\label{alg:VAW}
\end{algorithm}

We prove the following guarantee for VAW through an (almost) purely combinatorial argument, rather than
the standard elliptical-potential analysis.

\begin{theorem}\label{thm:VAW-smooth}
Under \cref{asmp:smooth}, assuming $y_t\in[-1,1]$ for all $t\in[T]$, the VAW predictor in
\cref{alg:VAW} achieves the following regret for any $\delta \in (0, 1)$, \whp:
\[
\Reg(T)\lesssim \left(\sqrt{d\Ccov T\log(T/\delta)}+\log(1/\delta)\right).
\]
\end{theorem}

We briefly indicate how the proof goes. It is standard to reduce the regret analysis of VAW-type
algorithms to upper bounding the sum $\sumt \xt\tp \covt\inv \xt$; in the usual analysis this leads to
a dependence on the magnitudes of the $x_t$'s. In contrast, we use smoothed-analysis ideas. To control
$\sumt \xt\tp \covt\inv \xt$, we consider the longest subsequence
$x_{t_1},\ldots,x_{t_k}$ such that $\nrm{x_{t_j}}_{\cov_{t_j}^\dagger}^2\ge r$ for a fixed threshold
$r>0$, and then use the coupling result from \cref{lem:smooth-to-iid} to reduce to an i.i.d.\ sequence.
Compare this with the standard VAW bound \eqref{eq:vaw-regret}, where a regularization parameter
$\lambda$ is essential (one cannot take $\lambda\to 0$ without making the bound trivial), and the
resulting regret bound depends explicitly on $\max_{t\le T}\|x_t\|$.

\section{Bounding self-normalized martingales via regret analysis: high probability results}
\label{sec:selfnorm-via-regret}

We are now back in the setup of self-normalized martingales, and we re-derive a self-normalized
concentration inequality by combining (i) a {deterministic} regret bound for an online regression
algorithm and (ii) a {stochastic} exponential supermartingale coming from the conditional sub-Gaussian
noise assumption. Such a reduction is standard and is a key tool in \citep{rakhlin2017equivalence}.

Let $(\cF_t)_{t\ge 0}$ be a filtration such that, for each round $t$, the covariate $X_t$ and the
prediction $\hat y_t$ are $\cF_{t-1}$-measurable, and $Y_t$ is then revealed.
Assume that $(Y_t)_{t\ge 1}$ is a martingale difference sequence with conditionally sub-Gaussian
increments: for some $\sigma>0$,
\begin{equation}\label{eq:cond-subg-sec}
    \mathbb{E}\left[Y_t\mid \cF_{t-1}\right]=0
    \qquad\text{and}\qquad
    \mathbb{E}\left[\exp(\alpha Y_t)\mid \cF_{t-1}\right]\le \exp\left(\frac{\alpha^2\sigma^2}{2}\right),
\end{equation}
for all $\alpha \in \mathbb{R}$ and $t \ge 1$.
Fix a deterministic $0\prec \Gamma\in\R^{d\times d}$ and define
$S_t \ldef \sum_{i=1}^t X_i Y_i$ and $\cov_t \ldef \sum_{i=1}^t X_i X_i^\top$.
As above, completion of squares yields
\begin{equation}\label{eq:selfnorm-identity-sec}
    \sum_{t=1}^T Y_t^2
    -\inf_{\theta\in\R^d}\left\{\sum_{t=1}^T\bigl(\tri{\theta,X_t}-Y_t\bigr)^2+\theta^\top\Gamma\theta\right\}
    \;=\;\nrm{S_T}_{(\cov_T+\Gamma)^{-1}}^2,
\end{equation}
which equivalently, for \emph{any} sequence of predictions $(\hat y_t)_{t=1}^T$, can be rewritten as
\begin{equation}\label{eq:regret-decomp-sec}
    \nrm{S_T}_{(\cov_T+\Gamma)^{-1}}^2
    =
    \Reg_T(\hat y)
    +\sum_{t=1}^T\bigl(2\hat y_t Y_t-\hat y_t^2\bigr),
\end{equation}
where the last term is the one we will control stochastically using \eqref{eq:cond-subg-sec}.
The next simple result shows that this term admits a sharp high-probability bound.

\begin{lemma}\label{lem:exp-supermg}
Under the sub-Gaussian assumption \eqref{eq:cond-subg-sec}, for any predictable sequence
$(\hat y_t)_{t=1}^T$, the process
$
M_t \ldef \exp\left(\frac{1}{2\sigma^2}\sum_{i=1}^t \bigl(2\hat y_i Y_i-\hat y_i^2\bigr)\right)
$
is a nonnegative supermartingale with $\mathbb{E}[M_t]\le 1$ for all $t$.
In particular, with probability at least $1-\delta$,
\[
\sum_{t=1}^T\bigl(2\hat y_t Y_t-\hat y_t^2\bigr) \le 2\sigma^2\log(1/\delta),
\]
and the same bound extends to stopping times.
\end{lemma}

Finally, we combine the high-probability regret bound under smoothness with the generic reduction of
\Cref{eq:regret-decomp-sec} to obtain a self-normalized concentration inequality without
any explicit bound on $\|X_t\|$ and without introducing a positive definite regularization matrix $\Gamma$ in the
self-normalization.

\begin{theorem}[Self-normalized martingales under smoothness]\label{thm:selfnorm-smooth}
Let $(\cF_t)_{t\ge 0}$ be a filtration such that for each $t\le T$, the covariate $X_t$ is
$\cF_{t-1}$-measurable and $Y_t$ is then revealed.
Assume that $(Y_t)_{t=1}^T$ is a martingale difference sequence satisfying \eqref{eq:cond-subg-sec} with
parameter $\sigma$.
Assume also that $(X_t)_{t=1}^T$ satisfies the smoothness condition of Assumption~\ref{asmp:smooth} with
parameter $\Ccov\ge 1$.
Define $S_T\ldef \sum_{t=1}^T X_t Y_t$ and $\cov_T\ldef \sum_{t=1}^T X_t X_t^\top$.
Then for any $\delta\in(0,1)$, with probability at least $1-\delta$,
\begin{equation}\label{eq:selfnorm-smooth-final}
    \nrm{S_T}_{\cov_T^\dagger}^2 \lesssim\sigma^2\Bigl(\sqrt{d\,\Ccov\,T\,\log(2T/\delta)}+\log(2/\delta)\Bigr).
\end{equation}
\end{theorem}

Importantly, compared to the canonical bound \citep{abbasi2011improved}, namely under the sub-Gaussian
assumption \eqref{eq:cond-subg-sec} but without smoothness, for any positive definite $\Gamma$,
\[
\nrm{S_T}_{(\cov_T + \Gamma)^\dagger}^2 \le \sigma^2\left(
\log \left(\frac{\det(\cov_T+\Gamma)}{\det(\Gamma)}\right)+2\log(1/\delta)\right),
\]
the right-hand side of \eqref{eq:selfnorm-smooth-final} has no explicit dependence on $\max_t\|X_t\|$ and
no regularization matrix $\Gamma$ is needed in the self-normalized quantity.
Moreover, the base measure $\mu$ from Assumption~\ref{asmp:smooth} is only used for the analysis: even for the regret analysis of \Cref{thm:VAW-smooth} the
algorithm itself does not require knowing $\mu$, and the regret bound is used purely as a tool to prove
concentration.

\arxiv{
  \section*{Acknowledgments}
  We acknowledge support from AFOSR through award 
FA9550-25-1-0375, Simons Foundation and the NSF through awards DMS-2031883 and PHY-2019786, and DARPA AIQ award.
}
\bibliography{ref.bib}

\clearpage

\renewcommand{\contentsname}{Contents of Appendix}
\addtocontents{toc}{\protect\setcounter{tocdepth}{2}}
{\hypersetup{hidelinks}
\tableofcontents
}

\appendix  

\section{Proofs from \cref{sec:prlim}}

\subsection{Proof of Lemma \ref{lem:rakhlinsridharan}}

When the ``centering'' in \cite[Lemma 4]{rakhlin2014nonparametric} is incorporated into the $X$-process, the upper bound reads as 
$$\sup_{X} \En\sup_{\theta\in\RR^d} \sum_{t=1}^T 4\varepsilon_t \tri{(\theta,1), X_t} - \tri{(\theta,1), X_t}^2 $$
where $X=(X_t)$ is an $\RR^{d+1}$-valued tree with $X_t[d+1]\in[-1,1]$. We over-bound by choosing a $(d+1)$-dimensional $\theta$ in the supremum. \jmlrQED

\section{Proofs from \cref{sec:full-adversary}}
\subsection{Proof of \cref{thm:selfnorm-exp}}

We note that 
\[
S_{T+1}=S_T+\eps_{T+1}X_{T+1},
\qquad
V_{T+1}=V_T+X_{T+1}^2.
\]
Conditioning on $\cF_T$ and expanding,
\begin{align*}
\mathbb{E}\!\left[\exp(cR_{T+1})\mid \cF_T\right]
&=
\mathbb{E}\!\left[\exp\!\left(\frac{c(S_T+\eps_{T+1}X_{T+1})^2}{V_T+X_{T+1}^2}\right)\Bigm|\cF_T\right]\\
&=
\exp\!\left(\frac{cS_T^2+cX_{T+1}^2}{V_T+X_{T+1}^2}\right)\,
\mathbb{E}\!\left[\exp\!\left(\frac{2c\,\eps_{T+1}S_TX_{T+1}}{V_T+X_{T+1}^2}\right)\Bigm|\cF_T\right].
\end{align*}
Introduce
\[
u:=\frac{V_T}{V_T+X_{T+1}^2}\in[0,1],
\qquad
1-u=\frac{X_{T+1}^2}{V_T+X_{T+1}^2},
\qquad
\frac{S_T^2}{V_T+X_{T+1}^2}=u\frac{S_T^2}{V_T}=uR_T.
\]
Also define
\[
a := \frac{2c\,S_TX_{T+1}}{V_T+X_{T+1}^2}.
\]
By Hoeffding's lemma for a Rademacher variable,
$\mathbb{E}[\exp(\eps_{T+1}a)\mid \cF_T]\le \exp(a^2/2)$, hence
\[
\mathbb{E}\!\left[\exp(cR_{T+1})\mid \cF_T\right]
\le
\exp\!\left(cuR_T + c(1-u) + \frac{a^2}{2}\right).
\]
Next,
\[
\frac{a^2}{2}
=
\frac{1}{2}\cdot \frac{4c^2 S_T^2X_{T+1}^2}{(V_T+X_{T+1}^2)^2}
=
2c^2\cdot \frac{S_T^2}{V_T}\cdot \frac{V_TX_{T+1}^2}{(V_T+X_{T+1}^2)^2}
=
2c^2 R_T\,u(1-u).
\]
Therefore,
\begin{align}
\mathbb{E}\!\left[\exp(cR_{T+1})\mid \cF_T\right]
&\le
\exp\!\left(cuR_T + c(1-u) + 2c^2R_T\,u(1-u)\right) \notag\\
&\leq \exp\!\left(cuR_T + c(1-u) + 2c^2R_T\,(1-u)\right).  \notag
\end{align}
where the second inequality is because $2c^2R_T(1-u)^2  \geq 0$.
Now split into two cases.

\medskip
\noindent\emph{Case 1: $R_T\le \frac{1}{1-2c}$.} Then we have
\begin{align*}
    \exp\!\left(cuR_T + c(1-u) + 2c^2R_T\,(1-u)\right) &= \exp\!\left(c + 2c^2R_T+  cuR_T -cu - 2c^2R_Tu\right)\\
    &= \exp\!\left(c + 2c^2R_T+  cu( (1-2c)R_T-1)\right)\\
    &\leq \exp\!\left(c + 2c^2R_T\right)\\
    &\leq \exp\!\Bigl(\frac{c}{1-2c}\Bigr) = C_0,
\end{align*}
where both inequalities are due to the condition of case 1.

\medskip
\noindent\emph{Case 2: $R_T> \frac{1}{1-2c}$.} Then $c(1-u)(1-(1-2c)R_T)\leq 0$. Reorgnizing the terms, we have $cuR_T + c(1-u) + 2c^2R_T(1-u)\leq cR_T$. This implies
\begin{align*}
    \exp\!\left(cuR_T + c(1-u) + 2c^2R_T\,(1-u)\right) \leq \exp\!\left(cR_T\right).
\end{align*}

Combining both cases yields the one-step inequality
\[
\mathbb{E}\!\left[\exp(cR_{T+1})\mid \cF_T\right]
\le
\exp(cR_T)+C_0.
\]
Taking expectations and iterating gives
\[
\mathbb{E}\exp(cR_T)
\le
\mathbb{E}\exp(cR_1) + (T-1)C_0
\le
TC_0,
\]
since $\exp(cR_1) = \exp(c) \le \exp(c/(1-2c))=C_0$.
Finally, by Jensen's inequality,
$c\,\mathbb{E}[R_T]\le \log \mathbb{E}\exp(cR_T)\le \log(TC_0)$, proving the stated bound.
\jmlrQED

\subsection{\pfref{prop:lower-bound-1d}}

Let $n=\floor{\frac12\log T}$ and $K=\floor{T/n}$. We set $M=\frac{2T}{n}$.

Consider the following  sequence dyadic martingale difference sequence. Set $X_1=1$.
\begin{itemize}
    \item For $t=jn+1$, we set $X_t=M\cdot X_{(j-1)n+1}$ if there exists $\ell\in[(j-1)n+1,jn]$ such that $\eps_{\ell}=-1$. Otherwise we set $X_{t}=0$.
    \item For $t\in(jn+1,(j+1)n]$, we set $X_t=X_{jn+1}$.
\end{itemize}
Let $i$ be the first index such that $\eps_\ell=1\forall \ell\in[in+1,(i+1)n]$, and if no such index exists we set $i=K+1$. Then, if $i\leq K$, we can bound $\abs{S_T-M^in}\leq M^{i-1}T$ and $V_T\leq \frac{nM^{2i}}{1-M^{-1}}$. Therefore $\frac{S_T^2}{V_T}\geq \frac{(n-T/M)^2}{2n}\geq \frac{n}{8}$. On the other hand, we have
\begin{align*}
    \PP(i=K+1)\leq&~ \PP(\forall 0\leq j<K, \exists \ell\in[jn+1,(j+1)n], \eps_\ell\neq 1 ) \\
    \leq&~ (1-2^{-n})^K\leq e^{-2^{-n}K}\leq 1-c_0,
\end{align*}
where $c_0>0$ is an absolute constant. This implies $\En[S_T^2/V_T]\geq \frac{c_0}{8}n=\Omega(\log T)$.
\jmlrQED

\subsection{\pfref{thm:1d-reg-upper}}

\begin{algorithm}[h]
\begin{algorithmic}[1]
\REQUIRE Parameter $M>1$, round $T$, subroutine sequence $\crl{\Alg_k}_{k\in\ZZ}$.
\STATE Initialize $k=-\infty$, $\cD=\emptyset$.
\FOR{$t=1,2,\cdots,T$}
\IF{$\abs{x\ind{t}}\geq M^{k+1}$}
\STATE Update $k\leftarrow \floor{\log_M \abs{x\ind{t}}}$.
\STATE Initialize $\Alg_k$ and set $\cD=\emptyset$. 
\ENDIF
\STATE Predict $\yhat\ind{t}=\Alg_k(\cD, x\ind{t})$.
\STATE Receive $y\ind{t}$ and update $\cD\leftarrow \cD\cup (x\ind{t},y\ind{t})$.
\ENDFOR
\end{algorithmic}
\caption{Meta algorithm}
\label{alg:meta}
\end{algorithm}

The above algorithm utilizes a sequence of subroutines $\crl{\Alg_k}_{k\in\ZZ}$ and additionaly the trivial algorithm $\Alg_{\infty}$ that always predicts $\yhat\ind{t}=0$.
We can then decompose the regret of \cref{alg:meta} to the regret of each subroutine $\Alg_k$. 
\begin{assumption}\label{asmp:subroutine}
For any $k\in\ZZ$ and $n\leq T$, on any sequence $(x\ind{1},y\ind{1},\cdots,x\ind{n},y\ind{n})$ such that $\max_{i\in[n]} \abs{x\ind{i}}\in [M^k,M^{k+1})$, the algorithm $\Alg_k$ achieves $\Reg\leq R_T$ and additionally
\begin{align*}
    \sumt (\yhat\ind{t}-y\ind{t})^2-\sumt y\ind{t}^2 \leq \beta_T.
\end{align*}
\end{assumption}

\begin{lemma}\label{lem:decomp}
Under \cref{asmp:subroutine}, \cref{alg:meta} achieves
\begin{align*}
    \Reg(T)\leq T\beta_T+2R_T+\frac{2T^2}{M}.
\end{align*}
\end{lemma}

\begin{proof}
    \newcommand{\kstar}{k^\star}
Let $\cK\subset \crl{-\infty} \cup \ZZ$ be the set of index $k$ such that $\Alg_k$ is executed, and
suppose that $\Alg_k$ is executed on the time interval $T_k$. When $x\ind{1}=\cdots=x\ind{T}=0$ there is nothing to prove. Otherwise, we note that
\[\hth\ldef \argmin_{\theta\in \RR}\sumt (\tri{\theta, x\ind{t}}-y\ind{t})^2 = \frac{\sumt y\ind{t} x\ind{t}}{\sumt x\ind{t}^2}.\]
By Cauchy–Schwarz's inequality and $\max_t x_t\leq \sqrt{\sumt x_t^2}$, we have 
\begin{align*}
    \max_t|x_t| \cdot \prn*{\sumt y\ind{t} x\ind{t}}  \leq \max_t|x_t|\cdot  \sqrt{\sumt y_t^2} \sqrt{\sumt x_t^2} \leq \sqrt{T}~ \sumt x_t^2.
\end{align*}
In particular, $\abs{\hth}\leq \frac{\sqrt{T}}{\max_t \abs{x\ind{t}}}$.
Therefore, we let $\kstar$ be the maximum of $\cK$, and then  $\abs{\hth}\leq \frac{\sqrt{T}}{M^{\kstar}}$. We can then decompose
\begin{align*}
     \Reg=&~ \sumt (\yhat\ind{t}-y\ind{t})^2-\sumt (\hth x\ind{t}-y\ind{t})^2 \\
     =&~ \sum_{k\in\cK} \sum_{t\in T_k} \brk*{ (\yhat\ind{t}-y\ind{t})^2- (\hth x\ind{t}-y\ind{t})^2 }.
\end{align*}
Note that when $k\leq \kstar-2$, it holds that for any $t\in T_k$, $\abs{x\ind{t}}\leq M^{\kstar-1}$ and
\begin{align*}
    (\hth x\ind{t}-y\ind{t})^2
    \geq y\ind{t}^2 - 2 \hth x\ind{t} y\ind{t}
    \geq y\ind{t}^2 - \frac{2\sqrt{T}}{M}.
\end{align*}
Therefore,
\begin{align*}
     \Reg
     =&~ \sum_{k\in\cK} \sum_{t\in T_k} \brk*{ (\yhat\ind{t}-y\ind{t})^2- (\hth x\ind{t}-y\ind{t})^2 } \\
     \leq&~ \sum_{k\in\cK: k\leq \kstar-2} \sum_{t\in T_k} \brk*{ (\yhat\ind{t}-y\ind{t})^2- y\ind{t}^2 + \frac{2{\sqrt{T}}}{M} } + \sum_{k\in\cK: k\geq \kstar-1} \sum_{t\in T_k} \brk*{ (\yhat\ind{t}-y\ind{t})^2- (\hth x\ind{t}-y\ind{t})^2 } \\
     \leq&~ T\prn*{\beta_T+\frac{2\sqrt{T}}{M}}+2R_T,
\end{align*}
where we use $\sum_{t\in T_k} \brk*{ (\yhat\ind{t}-y\ind{t})^2- y\ind{t}^2}\leq \beta_T$ and $\sum_{t\in T_k} \brk*{ (\yhat\ind{t}-y\ind{t})^2- (\hth x\ind{t}-y\ind{t})^2 }\leq R_T$ by our assumption (because we also know $\max_{t\in T_k} \abs{x\ind{t}}\in[M^k,M^{k+1})$ unless $k=-\infty$). 
\end{proof}

\paragraph{Subroutines}
It remains to construct subroutines satisfying \cref{asmp:subroutine}. We first recall that Vovk-Azoury-Warmuth forecaster has the following guarantee. 
\begin{lemma}
On any sequence $(x\ind{1},y\ind{1},\cdots,x\ind{n},y\ind{n})$, the following rule
\begin{align}\label{eq:Vovk}
\begin{aligned}
    \hth\ind{t}=&~\argmin_{\theta\in\RR^d} \lambda \nrm{\theta}^2+  \tri{\theta, x\ind{t}}^2+\sum_{i<t} (\tri{\theta, x\ind{i}}-y\ind{i})^2, \\
    \yhat\ind{t}=&~\clip_{[-1,1]}\prn*{\tri{\hth\ind{t},x\ind{t}}},
\end{aligned}
\end{align}
achieves the following for any $\theta\in\RR^d$:
\begin{align*}
    \sumt[n] (\yhat\ind{t}-y\ind{t})^2- \sumt[n] (\tri{\theta, x\ind{t}}-y\ind{t})^2
    \leq \lambda\nrm{\theta}^2 + d\log\prn*{1+\frac{n\max_{t} \nrm{x\ind{t}}^2}{\lambda}}.
\end{align*}
In particular, when $d=1$, we have the following guarantee:
\begin{align*}
    \sumt[n] (\yhat\ind{t}-y\ind{t})^2- \inf_{\theta\in\Theta}\sumt[n] (\tri{\theta, x\ind{t}}-y\ind{t})^2
    \leq \frac{n\lambda}{\max_{t} x\ind{t}^2} + \log\prn*{1+\frac{n \max_{t} x\ind{t}^2}{\lambda}}.
\end{align*}
\end{lemma}

However, the forecaster \eqref{eq:Vovk} may not achieve good $\beta(n)$ bound. Therefore, we hedge it against the $0$ forecaster.
\begin{lemma}\label{lem:exp-2}
Consider the following two experts problem: For $t\geq 1$, expert 0 predicts $\yhat\ind{t,0}=0$ and expert 2 predicts $\yhat\ind{t,1}$ following \eqref{eq:Vovk}. The final prediction is given by
\begin{align*}
    \phat\ind{t}(j)\propto_{j\in \crl{0,1}} p\ind{t}(j) \exp\prn*{-\eta\sum_{i<t} (\yhat\ind{i,j}-y\ind{i})^2}, \qquad
    \yhat\ind{t}=\En_{j\sim \phat\ind{t}}[\yhat\ind{i,j}],
\end{align*}
where $p_0(1)=1-p_0(0)=\eps$. Then as long as $\eta\leq \frac{1}{8}$, it holds that
\begin{align*}
    &~ \sumt[n] (\yhat\ind{t}-y\ind{t})^2
    - \sumt[n] (\yhat\ind{t,1}-y\ind{t})^2\leq \frac{1}{\eta}\log\frac{1}{p_0(1)}=\frac{\log(1/\eps)}{\eta}, \\
    &~ \sumt[n] (\yhat\ind{t}-y\ind{t})^2-\sumt[n] y\ind{t}^2
    \leq \frac{1}{\eta}\log\frac{1}{p_0(0)}
    \leq \frac{\eps}{(1-\eps)\eta}.
\end{align*}
In particular, it holds that $\Reg\leq \frac{n\lambda}{\max_{t} x\ind{t}^2} + \log\prn*{1+\frac{n \max_{t} x\ind{t}^2}{\lambda}}+\frac{\log(1/\eps)}{\eta}$.
\end{lemma}

To summarize, we have the following corollary.
\begin{corollary}
For any $k\in\ZZ$ and $\beta\in(0,1]$, there exists a subroutine $\Alg_k$ (by choosing $\lambda=\frac{M^{2k}}{T}$, $\eta=\frac{1}{8}$, and $\eps=\frac{1}{16}\beta$ in \cref{lem:exp-2}) such that \cref{asmp:subroutine} holds with $\beta_T=\beta$ and $R_T\leq O(\log(TM/\beta))$. 

In other words, on any sequence $(x\ind{1},y\ind{1},\cdots,x\ind{n},y\ind{n})$ such that $\max_{i\in[n]} \abs{x\ind{i}}\in [M^k,M^{k+1})$ and $n\leq T$, the subroutine achieves
\begin{align*}
    \sumt[n] (\yhat\ind{t}-y\ind{t})^2-\sumt[n] y\ind{t}^2
    \leq \beta, 
    \qquad
    \Reg\leq O(\log(TM/\beta)).
\end{align*}
\end{corollary}

In particular, we can choose $\beta=\frac1T$ and $M=T^2$, and by \cref{lem:decomp}, \cref{alg:meta} can be suitably instantiated such that it achieves
\begin{align*}
    \Reg\leq O(\log T).
\end{align*}
As a remark, the dependence $\log(1/\beta)$ is crucial for this regret bound, and it can be regarded as the ``price of super-efficiency'' against $0$.

\subsection{\pfref{prop:lower}}
The martingale is constructed as follows. Let $\varepsilon_1,\ldots,$ be an i.i.d. sequence of Rademacher random variables, and we will define $x_t=X_t(\varepsilon_1,\ldots,\varepsilon_{t-1})$ below. Let $S_t=\sum_{i=1}^t \varepsilon_i x_i$ and $\covt=\sum_{i=1}^t x_i x_i^\top$. Consider $\lcovt=\covt+\id\succ \covt$. Observe that under our construction,
\begin{align*}
    \En\brk*{ \nrm{\St[t+1]}_{\lcovt[t+1]^{-1}}^2}
    =\En\brk*{ \nrm{\St[t]}_{\lcovt[t+1]^{-1}}^2+\nrm{\xt[t+1]}_{\lcovt[t+1]^{-1}}^2 }.
\end{align*}
Using 
\begin{align*}
    \lcovt[t+1]^{-1}=\lcovt^{-1}-\frac{\lcovt^{-1}x\ind{t+1}x\ind{t+1}^\top \lcovt^{-1}}{1+\nrm{x\ind{t+1}}_{\lcovt^{-1}}^2},
\end{align*}
we have 
\begin{align*}
    \En\brk*{ \nrm{\St[t+1]}_{\lcovt[t+1]^{-1}}^2}
    =\En\brk*{ \nrm{\St[t]}_{\lcovt^{-1}}^2+\frac{\nrm{\xt[t+1]}_{\lcovt^{-1}}^2-\tri{\xt[t+1], \lcovt^{-1}\St}^2}{1+\nrm{x\ind{t+1}}_{\lcovt^{-1}}^2} }.
\end{align*}

We define $x\ind{t}$ recursively: $x\ind{1}$ is an arbitrarily fixed vector with norm $r>0$, and for $t\geq 1$, 
\begin{align}
    &~ x\ind{t+1}=r\lcovt^{1/2} e\ind{t}, \qquad \text{where $e\ind{t}$ is a unit vector such that } \tri{e\ind{t}, \lcovt^{-1/2} S\ind{t}}=0.
\end{align}
Then, it holds that
\begin{align*}
    \En\brk*{ \nrm{\St[t+1]}_{\lcovt[t+1]^{-1}}^2 }
    =\En\brk*{ \nrm{\St[t]}_{\lcovt^{-1}}^2+\frac{\nrm{\xt[t+1]}_{\lcovt^{-1}}^2}{1+\nrm{x\ind{t+1}}_{\lcovt^{-1}}^2} }
    =\En\brk*{ \nrm{\St[t]}_{\lcovt^{-1}}^2 } + \frac{r^2}{r^2+1}.
\end{align*}
Hence, since $S\ind{T}\in \mathrm{Range}(\covt[T])$, it is clear that 
$\En[\nrm{S\ind{T}}_{\covt[T]^{\dagger}}^2] \geq \En[\nrm{S\ind{T}}_{{\lcovt[T]^{-1}}}^2]=\frac{Tr^2}{r^2+1}$, and the proof is then completed by choosing $r=\frac{1}{\eps}$. 
\jmlrQED

\section{Proofs from \cref{sec:upper}}

\subsection{Proof of Lemma \ref{lem:exp-supermg}}
Fix $t\ge 1$. Since $\hat y_t$ is $\cF_{t-1}$-measurable, using \eqref{eq:cond-subg-sec} with
$\alpha=\hat y_t/\sigma^2$ gives
\begin{align*}
    \mathbb{E}\left[\exp\left(\frac{1}{2\sigma^2}(2\hat y_t y_t-\hat y_t^2)\right)\,\middle|\,\cF_{t-1}\right]
    &=
    \exp\left(-\frac{\hat y_t^2}{2\sigma^2}\right)\,
    \mathbb{E}\left[\exp\left(\frac{\hat y_t}{\sigma^2}y_t\right)\,\middle|\,\cF_{t-1}\right]\\
    &\le
    \exp\left(-\frac{\hat y_t^2}{2\sigma^2}\right)\,
    \exp\left(\frac{\hat y_t^2}{2\sigma^2}\right)
    =1.
\end{align*}
Multiplying by $M_{t-1}$ and taking conditional expectations yields
$\mathbb{E}[M_t\mid \cF_{t-1}]\le M_{t-1}$, so $(M_t)$ is a nonnegative supermartingale and hence
$\mathbb{E}[M_t]\le \mathbb{E}[M_0]=1$.
Finally, the proof follows from Ville's inequality for nonnegative supermartingales. \jmlrQED

We will make use of the following upper bound of VAW. 
\begin{proposition}\label{prop:VAW}
Suppose that $y_t\in[-1,1]$ deterministically. Then
\cref{alg:VAW} achieves
\begin{align*}
    \sumt (\yhat\ind{t}-y\ind{t})^2-\inf_{\theta\in\RR^d} \sumt (\tri{\theta, x\ind{t}}-y\ind{t})^2\leq \sumt \xt\tp \covt\inv \xt.
\end{align*}
More generally, suppose that $y_t$ satisfies $\En[e^{y_t^2/m^2}\mid \cF_{t-1}]\leq e$. Then it holds that \whp,
\begin{align*}
    \sumt (\yhat\ind{t}-y\ind{t})^2-\inf_{\theta\in\RR^d} \sumt (\tri{\theta, x\ind{t}}-y\ind{t})^2\leq m^2\sumt \xt\tp \covt\inv \xt+m^2\log(1/\delta).
\end{align*}
\end{proposition}

\begin{proof}[\pfref{prop:VAW}]
Note that $\yht$ does not depend on the choice of $\htht$. Therefore, we only need to consider $\htht=\covt\inv \St[t-1]$, where $\St[t-1]=\sum_{i<t} \xt[i]\yt[i]$. Further, we know $\hth=\covt[T]\inv\St[T]\in\argmin _{\theta\in\RR^d} \sumt (\tri{\theta, x\ind{t}}-y\ind{t})^2$.
Then, we can calculate
\begin{align*}
    \sumt (\yhat\ind{t}-y\ind{t})^2-\sumt (\tri{\hth, x\ind{t}}-y\ind{t})^2
    =&~ \sumt \brk*{ \prn{ \xt\tp \covt\inv \St[t-1]}^2 - 2\yt \xt\tp \covt\inv \St[t-1] } + \St[T]\tp \covt[T]\inv \St[T].
\end{align*}
We also have
\begin{align*}
    \St\tp \covt\inv \St=\yt^2 \xt\tp \covt\inv \xt + 2\yt\xt \covt\inv \St[t-1]+ \St[t-1]\tp \covt\inv \St[t-1].
\end{align*}
Further, we have the following basic inequality: For PSD matrix $\cov$, $v\in \spn(\cov)$, and any $w$, it holds that
\begin{align*}
    v(\cov+ww\tp)\inv v\leq v\cov\inv v- (w\tp (\cov+ww\tp)\inv v)^2.
\end{align*}
When $w\in\spn(\cov)$ this is straight-forward by restricting to the subspace $\spn(\cov)$ where $\cov$ becomes invertible. Otherwise, we can consider $h=(\cov+ww\tp)\inv v$, and then $v=\cov h+w\tri{w,h}$, and using $v\in\spn(\cov)$ implies $\tri{w,h}=0$. Plugging this in, and we can see the equality holds.

Now, combining the inequalities above, we know 
\begin{align*}
    \St\tp \covt\inv \St\leq \yt^2 \xt\tp \covt\inv \xt + 2\yt\xt \covt\inv \St[t-1]+ \St[t-1]\tp \covt[t-1]\inv \St[t-1]-\prn{\xt\tp \covt\inv \St[t-1]}^2.
\end{align*}
Applying this inequality recursively, it holds
\begin{align*}
    \St[T]\tp \covt[T]\inv \St[T]\leq \sumt \brk*{ -\prn{ \xt\tp \covt\inv \St[t-1]}^2 + 2\yt \xt\tp \covt\inv \St[t-1]+\yt^2 \xt\tp \covt\inv \xt }.
\end{align*}
Reorganizing yields 
\begin{align*}
    \sumt (\yhat\ind{t}-y\ind{t})^2-\inf_{\theta\in\RR^d} \sumt (\tri{\theta, x\ind{t}}-y\ind{t})^2\leq \sumt y_t^2\xt\tp \covt\inv \xt,
\end{align*}
and the first inequality follows immediately. For the second upper bound, we use the fact that $\En[e^{\lambda y^2/m^2}|\cF_{t-1}]\leq e^\lambda$ for $\lambda\in[0,1]$, and hence using this inequality recursively, we derive
\begin{align*}
    \En\brk*{\exp\prn*{ \frac{1}{m^2}\sumt y_t^2\xt\tp \covt\inv \xt-\sumt \xt\tp \covt\inv \xt }}\leq 1.
\end{align*}
By Markov's inequality, the desired upper bound follows.
\end{proof}

To provide a upper bound for VAW, it remains to upper bound the quantity $\sumt \xt\tp \covt\inv \xt$. This is nontrivial because the matrix $\covt$ can be ill-conditioned, and (as we have shown in \cref{sec:lower}) this sum can be $\Omega(T)$ without smoothness. In the following, we bound this quantity by a ``combinatorial dimension'' of the sequence $x$, and then apply the coupling trick and the backward analysis technique.

\paragraph{Step 1: Connect to combinatorial quantities of the sequence} 
To proceed, we introduce the notion of ``bad'' subsequence. For any sequence $z=(z\ind{1},\cdots,z\ind{k})$, we define $\cov(z)\ldef \sum_{i\leq k}z\ind{i}z\ind{i}\tp$. We call a sequence $z=(z\ind{1},\cdots,z\ind{k})$ \emph{$r$-bad} if for any $t\in[k]$, $\nrm{\zt}_{\cov(\zt[1:t])\inv}^2\geq r$.
We define $N(r;z)$ to be the length of the longest $r$-bad subsequence of $z$.

In the following, we bound the sum  $\sumt \xt\tp \covt\inv \xt$ by $N(r;x)$ for $r\in\crl{2^{-1},2^{-2},\cdots}$.
\begin{lemma}\label{lem:elliptical-to-comb}
For any sequence $x=(\xt[1],\cdots,\xt[T])$, it holds that
\begin{align*}
    \sumt \xt\tp \covt(x)\inv \xt \leq\int_{0}^1 N(r;x)dr \leq 1+\sum_{1\leq i\leq \log n} 2^{-i}N(2^{-i};x).
\end{align*}
\end{lemma}
\begin{proof}[\pfref{lem:elliptical-to-comb}]
    By definition, $\xt\tp \covt(x)\inv \xt\in[0,1]$, and hence
    \begin{align*}
        \xt\tp \covt(x)\inv \xt=\int_{0}^1 \indic{\xt\tp \covt(x)\inv \xt\geq r}dr.
    \end{align*}
    Therefore,
\begin{align*}
    \sumt \xt\tp \covt(x)\inv \xt 
    =\int_{0}^1 \sumt \indic{\xt\tp \covt(x)\inv \xt\geq r}dr.
\end{align*}
Note that $\sumt \indic{\xt\tp \covt(x)\inv \xt\geq r}\leq N(r;x)$, because if we consider all the indices $t_1<\cdots<t_k$ such that $\xt\tp \covt(x)\inv \xt\geq r$, then $(x_{t_1},\cdots,x_{t_k})$ is a $r$-bad sequence.
Hence, we have
\begin{align*}
    \sumt \xt\tp \covt(x)\inv \xt 
    =&~\int_{0}^1 \sumt \indic{\xt\tp \covt(x)\inv \xt\geq r}dr\\
    \leq&~ \int_{0}^1 N(r;x)dr
    \leq 1+\sum_{1\leq i\leq \log n} 2^{-i}N(2^{-i};x).
\end{align*}
The claim follows.

\end{proof}

\paragraph{Step 2: Coupling smooth sequence to a i.i.d sequence}
We invoke the following coupling lemma for smooth sequence~\citep{haghtalab2024smoothed}. The idea of this lemma is quite simple: Any smooth sequence can be generated by performing rejection sampling. We sketch the proof below. 
\begin{lemma}\label{lem:smooth-to-iid}
Suppose that \cref{asmp:smooth} holds. Then for any $\delta\in(0,1)$, $K\geq \Ccov\log(T/\delta)$, there exists a coupling between $x=(\xt[1],\cdots,\xt[T])$ with a sequence $z=(z_{1,1},\cdots,z_{1,K};\cdots;z_{T,1},\cdots,z_{T,K})$ such that:

(1) Marginally, $(z_{t,j})_{1\leq t\leq T, 1\leq j\leq K}\sim \mu$ are i.i.d random vectors from $\mu$.

(2) \Whp, it holds that $x_t\in \crl*{z_{t,i}: i=1,2,\cdots,K}$ for all $t\in[T]$.
\end{lemma}
Note that this lemma implies (\whp) $N(r;x)\leq N(r;z)$, and it remains to control $N(r;z)$ under the i.i.d sequence $z$. 

\begin{proof}[\pfref{lem:smooth-to-iid}]
Consider the randomness $z^{\infty}=(z_{t,j})_{1\leq t\leq T, j\geq 1}\sim \mu$ are i.i.d random vectors from $\mu$. Consider an environment $\Env$ that adopts the following protocol for each $t=1,2,\cdots,T$:
\begin{itemize}
    \item Given the history $\cH_{t-1}^x$, the environment fix the distribution $p_t(\cdot)=P_t(\cdot\mid \cH_{t-1}^x)$ and perform \emph{rejection sampling:}
    \item For $j=1,2,\cdots$, with probability $\frac{p_t(z_{t,j})}{\Ccov \mu(z_{t,j})}$, the environment set $x_t=z_{t,j}$ and break. Otherwise, the environment goes to the next step $j+1$.
\end{itemize}
By the guarantee of rejection sampling, we know that conditional on $\cH_{t-1}$, the vector $x_t$ is generated from the distribution $p_t=P_t(\cdot\mid \cH_{t-1}^x)$. Further, for any fixed $t\in[T]$, \whp~it holds that $x_t=z_{t,j}$ for some $j\leq \Ccov \log(1/\delta)$. Therefore, by union bound, the above construction gives a coupling between the sequence $x$ and the i.i.d sequence $z^{\infty}$, such that $x_t\in \crl*{z_{t,i}: i=1,2,\cdots,K}$ for all $t\in[T]$.
\end{proof}

\paragraph{Step 3: Bad i.i.d sequence is rare}
Finally, the problem is now reduced to bounding the probability that a i.i.d sequence $(z_1,\cdots,z_k)$ being $r$-bad, as we handle in the following proposition.

\newcommand{\zsig}{z\ind{\sigma}}
\begin{proposition}\label{prop:iid-bad}
Fix any $r\in(0,1]$. Suppose that $z=(z\ind{1},\cdots,z\ind{n})$ is a sequence of $n$ i.i.d vectors drawn from $\mu$. 
Then \whp~(over the randomness of $z$)
\begin{align*}
    N(r;z)\leq 3\sqrt{nd/r}+6\log(1/\delta).
\end{align*}
\end{proposition}

\begin{proof}[\pfref{prop:iid-bad}]
Fix any $k\geq 1$, we bound the probability $p\ldef \PP_{z\sim \mu}(N(r;z)\geq k)$. Note that $N(r;z)\geq k$ if and only if there exists a subset $I=\crl{i_1<\cdots<i_k}\subseteq [n]$ such that the sequence $z_I=(z_{i_1},\cdots,z_{i_k})$ is $r$-bad. Also, note that there are $\binom{n}{k}$ many such subsets. Therefore, we can bound
\begin{align}
    p\ldef \PP_{z\sim \mu}(N(r;z)\geq k)
    \leq \sum_{I\subseteq [n], \abs{I}=k} \PP_{z\sim \mu}(\text{$z_I$ is $r$-bad})
    =\binom{n}{k} p_0,
\end{align}
where we denote $p_0=\PP_{z\sim \mu}(\text{$(z_1,\cdots,z_k)$ is $r$-bad})$ and use the exchangeability of i.i.d random variables. 

In the following, we proceed to upper bound $p_0$. To this end, we consider the unordered multiset $\cS_i=\crl{z_1,\cdots,z_i}$ and the following random process:
\begin{align}
    \cS_n\to \cS_{n-1}\to\cdots\to \cS_1.
\end{align}
Note that this is a Markov chain, such that given $\cS_{i}$, the multiset $\cS_{i-1}$ is generated as first randomly select $z_i\sim \Unif(\cS_i)$, and the set $\cS_{i-1}=\cS_i\backslash \crl{z_i}$. Further, we note that $(z_1,\cdots,z_k)$ is $r$-bad if and only if
\begin{align}
    w_i\ldef z_i\prn*{\sum_{j\leq i} z_jz_j\tp}^\dagger z_i\geq r, \qquad \forall i\in[k].
\end{align}
Now, we can consider the backward expectation, where we define $\cH_i=(S_n,\cdots,S_i)$ to be the history up to step $i$:
\begin{align}
    \En[w_i\mid \cH_i]=\En[w_i\mid \cS_i]=&~\En\brk*{z_i\prn*{\sum_{z\in \cS_i} zz\tp}^\dagger z_i \mid \cS_i}=\En_{z_i\sim \Unif(\cS_i)}\brk*{ z_i\prn*{\sum_{z\in \cS_i} zz\tp}^\dagger z_i } \\
    =&~\tr\prn*{\frac{1}{i}\prn*{\sum_{z\in \cS_i} zz\tp}\prn*{\sum_{z\in \cS_i} zz\tp}^\dagger }\leq \frac{d}{i},
\end{align}
where we use $\tr(AA^\dagger)\leq d$ for any $d\times d$ positive semi-definite matrix $A$. In particular, this implies $\PP(w_i\geq r\mid \cH_i)\leq \min\crl*{1,\frac{d}{ri}}=:a_i$ for any $i\geq 1$. Now, we can bound
\begin{align}
    p_0=&~ \PP_{z\sim \mu}(\text{$(z_1,\cdots,z_k)$ is $r$-bad})
    =\PP_{z\sim \mu}(w_1\geq r,\cdots,w_k\geq r) \\
    =&~\En\brk*{ \indic{w_1\geq r,\cdots,w_k\geq r} }
    =\En\brk*{ \PP(w_1\geq r\mid \cH_1)\cdot \indic{w_2\geq r,\cdots,w_k\geq r} } \\
    \leq&~ a_1\En\brk*{ \indic{w_2\geq r,\cdots,w_k\geq r} }
    =a_1\En\brk*{ \PP(w_2\geq r\mid \cH_2)\cdot \indic{w_3\geq r,\cdots,w_k\geq r} } \\
    \leq&~ a_1a_2\En\brk*{ \indic{w_3\geq r,\cdots,w_k\geq r} }
    \leq\cdots \\
    \leq&~ a_1a_2\cdots a_k=\prod_{i=1}^k \min\crl*{1,\frac{d}{ri}}\leq \frac{d^k}{k!}\leq \prn*{\frac{ed}{rk}}^k,
\end{align}
where we use $k!\geq (k/e)^k$. Therefore, we can conclude that
\begin{align}
    p\leq \binom{n}{k}\prn*{\frac{ed}{rk}}^k\leq \prn*{\frac{e^2nd}{rk^2}}^k.
\end{align}
Then, as long as $k\geq 3\sqrt{nd/r}+6\log(1/\delta)$, it holds that $p=\PP_{z\sim \mu}(N(r;z)\geq k)\leq \delta$. This is the desired result.
\end{proof}

\newcommand{\oR}{\wb{R}}

\paragraph{Finalizing the proof}
Combining the results above, we can conclude that VAW achieves a sublinear regret.

\jqedit{
\begin{proposition}
\label{prop:smooth-VAW}
Under \cref{asmp:smooth}, with probability at least $1-\delta$:
\begin{align*}
    \sumt \xt\tp \covt\inv \xt \lesssim \sqrt{d\Ccov T\log(T/\delta)}+\log(1/\delta).
\end{align*}
\end{proposition}
\begin{proof}[\pfref{prop:smooth-VAW}]
By \cref{lem:elliptical-to-comb}, we can further upper bound
\begin{align}
    \sumt \xt\tp \covt\inv \xt\rdef \oR_T\leq 1+\sum_{1\leq i\leq \log n} 2^{-i}N(2^{-i};x).
\end{align}
Now, we take the coupling $\gamma$ constructed in \cref{lem:smooth-to-iid} (with $K=\ceil{\Ccov\log(2T/\delta)}$). We know that under $\gamma$ and a success event $\cE$ such that $\PP(\cE)\geq 1-\frac{\delta}{2}$, the sequence $z=(z_{t,j})_{1\leq t\leq T, 1\leq j\leq K}$ are i.i.d random vectors from $\mu$, and $x$ is a subsequence of $z$. This immediately implies that under $\cE$, we have $N(r;x)\leq N(r;z)$ for any $r\in[0,1]$. Finally, applying \cref{prop:iid-bad} to $r_i=2^{-i}$ for $1\leq i\leq \log_2 n$ and take a union bound, we know that \whp[\delta/2],
\begin{align}
    N(r_i;z)\leq 3\sqrt{nd/r_i}+6\log(4\log_2(n)/\delta), \qquad \forall 1\leq i\leq \log_2 n,
\end{align}
where $n=TK$.
Combining the inequalities above and taking union bound gives the desired result.
\end{proof}
}

\begin{proof}[\pfref{thm:VAW-smooth}]
By \cref{prop:VAW} and \cref{prop:smooth-VAW}, we have
\begin{align*}
    \Reg(T)&\ldef \sumt (\yhat\ind{t}-y\ind{t})^2-\inf_{\theta\in\RR^d} \sumt (\tri{\theta, x\ind{t}}-y\ind{t})^2\\
    &\leq \sumt \xt\tp \covt\inv \xt \\
    &\lesssim \sqrt{d\Ccov T\log(T/\delta)}+\log(1/\delta).
\end{align*}
\end{proof}

\subsection{Tightness of our analysis}

We argue that in the worst-case, the sum $\sumt \xt\tp \covt\inv \xt=\Omega(\sqrt{\Ccov T})$ even when $d=1$, demonstrating that our analysis is nearly tight.

\begin{lemma}
For any $C> 1$, there exists a smooth environment with $\Ccov\leq C$ such that
\begin{align*}
    \En\brk*{\sumt \frac{\xt^2}{1+\sum_{i\leq t} \xt[i]^2}} \geq \Omega(\sqrt{(C-1) T} \wedge T).
\end{align*}
\end{lemma}

\begin{proof}
Fix $1\leq n\leq \frac{(C-1)T}{4}+1$ and set $p=\min\crl*{\frac{C-1}{n},1}$.
We consider the following environment:
\begin{itemize}
    \item Initialize $k=1$.
    \item For $t=1,2,\cdots$: With probability $p$, set $\xt=2^k$ and set $k\leftarrow k+1$. If $k>n$, terminates (i.e., outputs $\xt=0$ afterwards). Otherwise, set $\xt=0$.
\end{itemize}
Then it is clear that the environment is $C$-smooth with measure $\mu$ given by $\mu(2^k)=\frac{p}{C}$ for $k\in[n]$ and $\mu(0)=1-\frac{np}{C}\geq \frac{1}{C}$. Further, when $\xt>0$, it must hold that $\frac{\xt^2}{1+\sum_{i\leq t} \xt[i]^2}\geq \frac12$. Therefore, we can lower bound
\begin{align*}
    \En\brk*{\sumt \frac{\xt^2}{1+\sum_{i\leq t} \xt[i]^2}}\geq \frac12\En[\min\crl*{n,L}],
\end{align*}
where $L\sim \mathrm{Binomial}(T,p)$. As long as $p\geq \frac{1}{4T}$, it is clear that there is a constant $c>0$ such that $\PP(L\leq cTp)\leq \frac12$, and hence $\En[\min\crl*{n,L}]\geq \frac{1}{4}\min\crl*{2n, cTp}$. Suitably choosing $n$ gives the desired lower bound.
\end{proof}

\subsection{Dyadic filtration is enough}\label{sec:dyadic}

\newcommand{\multiminimax}[1]{\ensuremath{\left\langle #1\right\rangle}}

\fccomment{
\begin{proof}[Proof of \cref{lem:dyadic-equivalence}]
The inequality $\cR_d^{\mathrm{dyadic}}(T) \leq \cR_d(T)$ is trivial. For the reverse, write $\cR_d(T)$ as
\begin{align}
\multiminimax{\sup_{x_t}\sup_{p_t\in\tilde{\Delta}}\En_{y_t\sim p_t}}_{t=1}^T\crl*{\sup_{\theta\in\RR^d} \sum_{t=1}^T 2y_t\tri*{\theta,x_t}- \tri*{\theta,x_t}^2}
\end{align}
where $\tilde{\Delta}$ is a set of zero-mean distributions supported on $[-1,1]$, and $\multiminimax{\cdot}_{t=1}^T$ stands for the repeated application of the operators. Introducing the tangent sequence $(y_t')$, the above expression is equal to
\begin{align}
&\multiminimax{\sup_{x_t}\sup_{p_t\in\tilde{\Delta}}\En_{y_t\sim p_t}}_{t=1}^T\crl*{\sup_{\theta\in\RR^d} \sum_{t=1}^T 2(y_t-\En_{p_t}[y_t'])\tri*{\theta,x_t}- \tri*{\theta,x_t}^2}\\
&\leq \multiminimax{\sup_{x_t}\sup_{p_t\in\tilde{\Delta}}\En_{y_t,y_t'\sim p_t}}_{t=1}^T\crl*{\sup_{\theta\in\RR^d} \sum_{t=1}^T 2(y_t-y_t')\tri*{\theta,x_t}- \tri*{\theta,x_t}^2} \\
&\leq \multiminimax{\sup_{x_t}\sup_{p_t\in\tilde{\Delta}}\En_{y_t,y_t'\sim p_t}\En_{\varepsilon_t}}_{t=1}^T\crl*{\sup_{\theta\in\RR^d} \sum_{t=1}^T 2\varepsilon_t(y_t-y_t')\tri*{\theta,x_t}- \tri*{\theta,x_t}^2} 
\end{align}
by Jensen's inequality and symmetry of the distribution of $y_t-y_t'$. Here $(\varepsilon_t)$ is a sequence of independent Radmeacher random variables. This quantity is at most
\begin{align}
&\multiminimax{\sup_{x_t}\sup_{y_t,y_t'\in[-1,1]}\En_{\varepsilon_t}}_{t=1}^T\crl*{\sup_{\theta\in\RR^d} \sum_{t=1}^T 2\varepsilon_t(y_t-y_t')\tri*{\theta,x_t}- \tri*{\theta,x_t}^2} \\
&= \multiminimax{\sup_{x_t}\sup_{\delta_t\in[-2,2]}\En_{\varepsilon_t}}_{t=1}^T\crl*{\sup_{\theta\in\RR^d} \sum_{t=1}^T 2\varepsilon_t\delta_t\tri*{\theta,x_t}- \tri*{\theta,x_t}^2}.
\end{align}
The final step is to write $\delta_t=\En z_t$ for a random variable $z_t$ supported on $\{-2,2\}$. Then Jensen's inequality leads to an upper bound of 
\begin{align}
&\multiminimax{\sup_{x_t}\sup_{\delta_t\in[-2,2]}\En_{z_t}\En_{\varepsilon_t}}_{t=1}^T\crl*{\sup_{\theta\in\RR^d} \sum_{t=1}^T 2\varepsilon_t z_t\tri*{\theta,x_t}- \tri*{\theta,x_t}^2} \\
&=\multiminimax{\sup_{x_t}\En_{\varepsilon_t}}_{t=1}^T\crl*{\sup_{\theta\in\RR^d} \sum_{t=1}^T 4\varepsilon_t z_t\tri*{\theta,x_t}- \tri*{\theta,x_t}^2} 
\end{align}
where the last equality holds because the variable $\varepsilon_t z_t$ has the same distribution as twice a Rademacher random variable, no matter what $\delta_t$ is. The last expression is now recognized as at most $2\cR_d^{\mathrm{dyadic}}(T)$.
\end{proof}
}

\begin{proof}[Proof of \cref{lem:dyadic-equivalence}]
The inequality $\cR_d^{\mathrm{dyadic}}(T)\le \cR_d(T)$ is immediate, since every dyadic martingale is
admissible in the definition of $\cR_d(T)$.

For the reverse inequality, let
\[
\Phi(s,V)\ldef \nrm{s}_{V^\dagger}^2,
\qquad s\in\RR^d,\quad V\succeq 0.
\]
Thus, for any admissible process $(X_t,Y_t)_{t=1}^T$,
\[
R_T=\Phi(S_T,\cov_T),
\qquad
S_t=\sum_{i\le t} Y_iX_i,\quad
\cov_t=\sum_{i\le t} X_iX_i^\top .
\]

Let $\tilde{\Delta}$ denote the set of all probability laws on $[-1,1]$ with mean zero. Define
recursively, for $t=T,T-1,\ldots,1$,
\[
F_{T+1}(s,V)\ldef \Phi(s,V),
\qquad
F_t(s,V)\ldef
\sup_{x\in\RR^d}\sup_{p\in\tilde{\Delta}}
\En_{Y\sim p}\!\left[F_{t+1}(s+Yx,\;V+xx^\top)\right].
\]

We first claim that
\[
\cR_d(T)\le F_1(0,0).
\]
Indeed, fix any admissible process $(X_t,Y_t)_{t=1}^T$, and we prove by backward induction on $t$ that
\[
\En\!\left[\Phi(S_T,\cov_T)\mid \mathcal G_{t-1}\right]
\le
F_t(S_{t-1},\cov_{t-1})
\qquad\text{a.s.}
\]
The case $t=T+1$ is tautological. Assuming the claim at time $t+1$, we have
\begin{align*}
\En\!\left[\Phi(S_T,\cov_T)\mid \mathcal G_{t-1}\right]
&=
\En\!\left[\En\!\left[\Phi(S_T,\cov_T)\mid \mathcal G_t\right]\middle| \mathcal G_{t-1}\right]\\
&\le
\En\!\left[F_{t+1}(S_t,\cov_t)\middle| \mathcal G_{t-1}\right]\\
&=
\En\!\left[F_{t+1}(S_{t-1}+Y_tX_t,\;\cov_{t-1}+X_tX_t^\top)\middle| \mathcal G_{t-1}\right].
\end{align*}
Conditionally on $\mathcal G_{t-1}$, the vector $X_t$ is fixed, while the conditional law of $Y_t$
belongs to $\tilde{\Delta}$ (because $|Y_t|\le 1$ a.s. and $\En[Y_t\mid \mathcal G_{t-1}]=0$). Hence the
last display is at most $F_t(S_{t-1},\cov_{t-1})$ by the definition of $F_t$. Taking expectations at
$t=1$ and then the supremum over all admissible processes yields $\cR_d(T)\le F_1(0,0)$.

Next we claim that, for every $t$ and every $V\succeq 0$, the map $s\mapsto F_t(s,V)$ is convex. This
is proved by backward induction on $t$. At time $T+1$, the claim is immediate since
\[
\Phi(s,V)=s^\top V^\dagger s
\]
and $V^\dagger\succeq 0$. If $F_{t+1}(\cdot,V')$ is convex for every $V'\succeq 0$, then for fixed
$x\in\RR^d$ and fixed $p\in\tilde{\Delta}$, the map
\[
s\mapsto \En_{Y\sim p}\!\left[F_{t+1}(s+Yx,\;V+xx^\top)\right]
\]
is convex, because translation and expectation preserve convexity. Taking the supremum over $x$ and $p$
shows that $F_t(\cdot,V)$ is convex as well.

Now define the dyadic Bellman recursion by
\[
G_{T+1}(s,V)\ldef \Phi(s,V),
\qquad
G_t(s,V)\ldef
\sup_{x\in\RR^d}\En_{\varepsilon}\!\left[G_{t+1}(s+\varepsilon x,\;V+xx^\top)\right],
\]
where $\varepsilon$ is a Rademacher random variable.

We show by backward induction that $F_t=G_t$ for all $t$. The terminal condition is clear. Assume
$F_{t+1}=G_{t+1}$. Fix $s\in\RR^d$, $V\succeq 0$, and $x\in\RR^d$, and define
\[
g_x(y)\ldef F_{t+1}(s+yx,\;V+xx^\top),
\qquad y\in[-1,1].
\]
Since $F_{t+1}(\cdot,V+xx^\top)$ is convex, the function $g_x$ is convex on $[-1,1]$. Therefore, for
every $y\in[-1,1]$,
\[
g_x(y)\le \frac{1+y}{2}g_x(1)+\frac{1-y}{2}g_x(-1).
\]
Taking expectation with respect to any $p\in\tilde{\Delta}$ and using $\En_{Y\sim p}[Y]=0$, we obtain
\[
\En_{Y\sim p}[g_x(Y)]
\le
\frac12 g_x(1)+\frac12 g_x(-1).
\]
Taking the supremum over $p\in\tilde{\Delta}$ and then over $x\in\RR^d$ gives
\[
F_t(s,V)
\le
\sup_{x\in\RR^d}
\frac{F_{t+1}(s+x,V+xx^\top)+F_{t+1}(s-x,V+xx^\top)}{2}
=
G_t(s,V).
\]
The reverse inequality is immediate, since the symmetric Rademacher law
$\frac12\delta_{-1}+\frac12\delta_{1}$ belongs to $\tilde{\Delta}$. Hence $F_t=G_t$ for all $t$.

Finally, let $H_t(s,V)$ denote the supremum over all dyadic trees from rounds $t,\ldots,T$ of the
expected terminal payoff starting from state $(s,V)$. Then $H_{T+1}=\Phi$, and $H_t$ satisfies the same
recursion as $G_t$:
\[
H_t(s,V)=
\sup_{x\in\RR^d}\En_{\varepsilon}\!\left[H_{t+1}(s+\varepsilon x,\;V+xx^\top)\right].
\]
Therefore $H_t=G_t$ for all $t$, and in particular
\[
G_1(0,0)=H_1(0,0)=\cR_d^{\mathrm{dyadic}}(T).
\]

Combining the previous steps, we conclude that
\[
\cR_d(T)\le F_1(0,0)=G_1(0,0)=\cR_d^{\mathrm{dyadic}}(T).
\]
Together with the trivial inequality $\cR_d^{\mathrm{dyadic}}(T)\le \cR_d(T)$, this yields $\cR_d^{\mathrm{dyadic}}(T)=\cR_d(T)$.
\end{proof}

\section{Proofs from \Cref{sec:selfnorm-via-regret}}

\subsection{Proof of \Cref{thm:selfnorm-smooth}}
By \Cref{lem:exp-supermg} and \eqref{eq:regret-decomp-sec}, \jqedit{taking $\Gamma$ in \eqref{eq:regret-decomp-sec} to $0$, we have} with probability at least $1-\delta/2$,
\begin{equation}\label{eq:selfnorm-smooth-step1}
    \nrm{S_T}_{\cov_T^\dagger}^2
    \le
    \Reg_T(\hat y) + 2\sigma^2\log(2/\delta).
\end{equation}

On the other hand, applying \jqedit{\Cref{prop:VAW,prop:smooth-VAW}}
to the VAW predictor with confidence
level $\delta/2$ gives, with probability at least $1-\delta/2$,
\begin{equation}\label{eq:selfnorm-smooth-step2}
    \Reg_T(\hat y)
    \lesssim
\sigma^2\Bigl(\sqrt{d\,\Ccov\,T\,\log(2T/\delta)}+\log(2/\delta)\Bigr).
\end{equation}
By the union bound, with probability at least $1-\delta$ both \eqref{eq:selfnorm-smooth-step1} and
\eqref{eq:selfnorm-smooth-step2} hold.\jmlrQED

\end{document}